\documentclass[11pt]{article}
\usepackage[letterpaper,top=1in,bottom=1in,left=1in,right=1in]{geometry}
\usepackage[parfill]{parskip}
\usepackage[round]{natbib}
\usepackage{amssymb,amsmath,amsthm,bbm,mathrsfs,mathtools,bm}
\mathtoolsset{showonlyrefs}
\newtheorem{prop}{Proposition}
\newtheorem{clm}{Claim}
\newtheorem{cor}{Corollary}
\newtheorem{lem}{Lemma}
\newtheorem{thm}{Theorem}
\newtheorem{rem}{Remark}
\newtheorem{ass}{Assumption}
\newcommand{\E}{\mathbb{E}} 
\renewcommand{\P}{\mathbb{P}} 
\newcommand{\R}{\mathbb{R}}

\newcommand{\N}{\mathbb{N}}
\newcommand{\trans}{\mathsf{T}}
\newcommand{\ind}{\mathbbm{1}}
\newcommand{\epsmis}{\varepsilon_{\text{mis}}}
\newcommand{\epstol}{\varepsilon_{\text{tol}}}
\newcommand{\epsrnd}{\varepsilon_{\text{rnd}}}
\newcommand{\epseff}{\varepsilon_{\text{eff}}}
\renewcommand{\S}{\mathcal{S}}
\newcommand{\A}{\mathcal{A}}
\DeclareMathOperator*{\argmax}{arg\,max}

\DeclarePairedDelimiter{\ceil}{\lceil}{\rceil}

\usepackage{hyperref}
\hypersetup{
    colorlinks=true,
    linkcolor=blue,
   citecolor=blue
}
\usepackage{enumitem}
\usepackage[linesnumbered,ruled,noend,nosemicolon]{algorithm2e}
\makeatletter
\renewcommand{\@algocf@capt@plain}{above}
\makeatother
\SetInd{0pt}{8pt}
\SetKw{Break}{break}
\usepackage[htt]{hyphenat}
\usepackage{authblk}
\usepackage[bottom]{footmisc}

\begin{document}

\title{Improved Algorithms for Misspecified Linear Markov Decision Processes}

\author[1]{Daniel Vial\thanks{Corresponding author. Email: \texttt{vial2@illinois.edu}.}}
\author[2]{Advait Parulekar}
\author[2]{Sanjay Shakkottai}
\author[1]{R.\ Srikant}
\affil[1]{University of Illinois at Urbana-Champaign}
\affil[2]{University of Texas at Austin}

\maketitle

\begin{abstract}
For the misspecified linear Markov decision process (MLMDP) model of \cite{jin2020provably}, we propose an algorithm with three desirable properties. \ref{itRegret} Its regret after $K$ episodes scales as $K \max \{ \epsmis , \epstol \}$, where $\epsmis$ is the degree of misspecification and $\epstol$ is a user-specified error tolerance. \ref{itComp} Its space and per-episode time complexities are bounded as $K \rightarrow \infty$. \ref{itPar} It does not require $\epsmis$ as input. To our knowledge, this is the first algorithm satisfying all three properties. For concrete choices of $\epstol$, we also improve existing regret bounds (up to log factors) while achieving either \ref{itComp} or \ref{itPar} (existing algorithms satisfy neither). At a high level, our algorithm generalizes (to MLMDPs) and refines the \texttt{Sup-Lin-UCB} algorithm, which \cite{takemura2021parameter} recently showed satisfies \ref{itPar} for contextual bandits. We also provide an intuitive interpretation of their result, which informs the design of our algorithm.
\end{abstract}

\section{Introduction} \label{secIntro}

Due to the large state spaces of modern reinforcement learning applications, practical algorithms must generalize across states. To understand generalization on a theoretical level, recent work has studied linear Markov decision processes (LMDPs), among other models (see Section \ref{secRelated}). The LMDP model assumes the  next-state distribution and reward are linear in known $d$-dimensional features, which enables tractable generalization when $d$ is small. Of course, this assumption most likely fails in practice, which motivates the misspecified LMDP (MLMDP) model. Here linearity holds up to some misspecification error $\epsmis$ in total variation and absolute value for the next-state distribution and reward, respectively (see Assumption \ref{assLin}).

In this work, we consider episodic finite-horizon MLMDPs, i.e., for each of $K$ episodes, the algorithm interacts with the MLMDP for $H$ steps. We assume the action space $\A$ is finite, though the state space $\S$ may be infinite. We measure performance in terms of regret $R(K)$, i.e., the additive loss in expected cumulative reward compared to the optimal policy (see \eqref{eqRegret}). We seek an algorithm with three basic properties:
\begin{enumerate}[leftmargin=*,align=left,itemsep=0pt,topsep=0pt,label=(P\arabic*),wide]
\item \label{itRegret} {\bf Asymptotically non-trivial regret:} Given a user-specified error tolerance $\epstol > 0$ and a failure probability $\delta > 0$, the algorithm should ensure that with probability at least $1-\delta$,
\begin{equation} 
\limsup_{K \rightarrow \infty} \frac{R(K)}{K} \leq \text{poly} \left( d , H , \log \tfrac{1}{\delta \epstol} \right) \max \{ \epsmis , \epstol \} . 
\end{equation}
Hence, in terms of $K$, $\epsmis$, and $\epstol$, we desire regret that scales as $K \max \{ \epsmis , \epstol \}$. The $K \epsmis$ term is unavoidable due to misspecification, while the $K \epstol$ term can be controlled by the user.

\item \label{itComp} {\bf Bounded complexity:} The space and per-episode time complexities should both be independent of $K$, so that the algorithm can be implemented for arbitrarily large $K$.

\item \label{itPar} {\bf Parameter free:} The algorithm should not require knowledge of the degree of misspecification $\epsmis$, which is unavailable in practice.
\end{enumerate}

\begin{table*}[t]
\begin{center}
\caption{Summary of MLMDP algorithms, where $\epsmis$ is the degree of misspecification and $\epstol$ is a user-chosen input. Regret bounds hide constants and log terms, except those that yield super-linear in $K$ terms. See Section \ref{secRelated} for other existing algorithms that require more restrictive assumptions or lack explicit implementation.} \label{tabSummary}
\centering
\begin{tabular}{lcccl}
\hline 
\rule[-1ex]{0pt}{2.5ex} {\bf Algorithm} & {\bf \ref{itRegret}} & {\bf \ref{itComp}} & {\bf \ref{itPar}}  & {\bf Regret bound}  \\ 
\hline 
\rule[-1ex]{0pt}{2.5ex} Ours, $\epstol \in (0,1)$$^*$ & Yes & Yes & Yes & $\sqrt{ d^3 H^4 \min \{ (\frac{d}{\epstol} )^2 , K \} } + \sqrt{H^3 K} +  \sqrt{d} H^2 K \max \{ \epsmis , \epstol \}$ \\ 
\hline 
\rule[-1ex]{0pt}{2.5ex} Ours, $\epstol = d / \sqrt{K}$ & No & No & Yes & $\sqrt{ d^3 H^4 K } + \sqrt{d} H^2  K \epsmis \sqrt{\log K}$ \\ 
\hline 
\rule[-1ex]{0pt}{2.5ex} Ours, $\epstol = \epsmis$ & Yes & Yes & No & $\sqrt{ d^3 H^4\min \{ (\frac{d}{\epstol} )^2 , K \} } + \sqrt{H^3 K} + \sqrt{d} H^2 K \epsmis$ \\ 
\hline 
\rule[-1ex]{0pt}{2.5ex} \cite{jin2020provably}  & No & No & No & $\sqrt{ d^3 H^4 K } + d H^2 K \epsmis \sqrt{\log K}$ \\ 
\hline 
\rule[-1ex]{0pt}{2.5ex} \cite{zanette2020frequentist}$^\dag$ & Yes & No & No & $\sqrt{ d^4 H^5 K } + d H^2 K \epsmis$ \\
\hline
\end{tabular} 
\end{center} 
{\quad \footnotesize $^*$To satisfy \ref{itRegret} and avoid a super-linear regret bound, this row assumes we choose $\epstol$ independent of $K$.} 

{\quad \footnotesize $^\dag$This paper reports regret with $\tilde{O}$ notation but (as far as we can tell) it does not hide super-linear terms.}

\end{table*}

These properties seem benign, but to the best of our knowledge, \textit{no existing algorithm satisfies all three}. We note in particular that \ref{itRegret} often fails because regret guarantees include $K \epsmis \text{polylog} ( K )$ terms. We would argue such super-linear (in $K$) bounds are asymptotically trivial, since the regret of any policy is linear in $K$ (for bounded rewards). However, even if $K$ is small and one can tolerate failure of the asymptotically-motivated \ref{itRegret} and \ref{itComp}, there are (essentially) no algorithms with $K \epsmis \text{polylog} ( K )$-type regret that satisfy \ref{itPar}. (An exception is model selection, which violates a stronger version of \ref{itPar}; see Section \ref{secRelated}.)

The situation is better if we restrict to misspecified linear contextual bandits (MLCBs), which are the special case $H=1$ (there, states are called contexts; we use the terms interchangeably). In particular, \cite{takemura2021parameter} showed that a \texttt{Sup-Lin-UCB} \citep{auer2002using,chu2011contextual} variant that satisfies \ref{itPar} has $K \epsmis \text{polylog}(K)$ regret. This result is important because the simpler \texttt{Lin-UCB} \citep{abbasi2011improved} can suffer $\Omega(K)$ regret when $\epsmis$ is unknown \citep{lattimore2020learning}. In light of this, and because \texttt{Sup-Lin-UCB} was originally motivated by technical issues \textit{seemingly unrelated to misspecification}, \cite{takemura2021parameter}'s result is also rather surprising. However, \texttt{Sup-Lin-UCB} is a complicated algorithm, so it was not intuitively clear (at least to us) why it should adapt to the misspecified setting better than \texttt{Lin-UCB}. As will be seen, one of our contributions is to provide a new interpretation of \texttt{Sup-Lin-UCB} that intuitively explains this. Furthermore, our interpretation is a key building block that leads to improved results for MLMDPs.

{\bf Contributions:} Our contributions are as follows.
\begin{itemize}[leftmargin=*,align=left,itemsep=0pt,topsep=0pt]

\item {\bf An intuitive \texttt{Sup-Lin-UCB} variant:} In Section \ref{secAlgCB}, we show that \texttt{Sup-Lin-UCB}'s success for MLCBs is not an accident; rather, it can be \textit{derived} from the perspective of misspecification. More precisely, we first propose an MLCB algorithm called \texttt{EXPL3}, which \underline{expl}icitly decides to \underline{expl}ore or \underline{expl}oit. \texttt{EXPL3} is simple and intuitive but requires $\epsmis$ as input to perform well. We overcome this requirement by constructing an intuitive ensemble of \texttt{EXPL3} algorithms, one for each possible $\epsmis$ value (see Proposition \ref{propEnsemble}). The ensemble closely resembles \texttt{Sup-Lin-UCB} and sheds light onto \cite{takemura2021parameter}'s result.

\item {\bf The \texttt{Sup-LSVI-UCB} algorithm:} In Section \ref{secAlgMDP}, we leverage the insights developed for MLCBs to design an MLMDP algorithm called \texttt{Sup-LSVI-UCB}.\footnote{``LSVI'' stands for ``least-squares value iteration'' and ``UCB'' stands for ``upper confidence bound.''} At a high level, \texttt{Sup-LSVI-UCB} combines our \texttt{Sup-Lin-UCB} variant with a backward induction procedure, analogous to how \texttt{LSVI-UCB} \citep{jin2020provably} was derived from \texttt{Lin-UCB}. However, because \texttt{Sup-Lin-UCB} is more complicated than \texttt{Lin-UCB}, we encounter new technical issues when generalizing from MLCBs to MLMDPs, which requires some new algorithmic ideas; see Remarks \ref{remRounding} and \ref{remOffPolicy}.

\item {\bf Improved guarantees:} In Section \ref{secResults}, we show that when the input $\epstol$ is chosen independent of $K$, \texttt{Sup-LSVI-UCB} is the first algorithm to satisfy \ref{itRegret}, \ref{itComp}, and \ref{itPar} (see Theorem \ref{thmMain}). If instead $\epstol = d / \sqrt{K}$, \texttt{Sup-LSVI-UCB} improves existing regret bounds (up to log factors), while simultaneously removing the requirement that $\epsmis$ is known (Corollary \ref{corUnknown}). Finally, if $\epsmis$ is known, we can set $\epstol = \epsmis$ to improve existing bounds while simultaneously avoiding unbounded complexity (Corollary \ref{corKnown}). See Table \ref{tabSummary}.
\end{itemize}

Finally, though somewhat orthogonal to our main results, we also revisit \texttt{Lin-UCB} for MLCBs. While it is known that this algorithm can be modified to obtain $K \epsmis \text{polylog}(K)$ regret when $\epsmis$ is known \citep{lattimore2020learning,jin2020provably}, we are not aware of any bounds that satisfy \ref{itRegret}. Hence, we propose a new modification that satisfies \ref{itRegret} and sharpens existing bounds when $K \geq \epsmis^{-2}$. See Section \ref{secLinUCB}.

\subsection{Related work} \label{secRelated}

{\bf MLMDP:} \cite{jin2020provably} proposed the aforementioned \texttt{LSVI-UCB}, which fails to satisfy all of \ref{itRegret}, \ref{itComp}, and \ref{itPar}. \cite{zanette2020frequentist} analyzed the Thompson sampling-based randomized LSVI \citep{osband2019deep}. Their regret bound uses $\tilde{O}(\cdot)$ notation, so \ref{itRegret} is a bit ambiguous, but we believe it holds. However, \ref{itComp} and \ref{itPar} again fail.

{\bf Regarding sample complexity:} Several papers (including \cite{jin2020provably}) provide sample complexity bounds, i.e., number of samples to learn an approximately optimal policy, though to our knowledge, all violate \ref{itPar}. On the other hand, such bounds yield algorithms that satisfy \ref{itComp} (and possibly \ref{itRegret}, though not for \cite{jin2020provably}): one can just fix the approximately optimal policy after finding it. However, this requires well-behaved initial states, e.g., fixed across episodes as in \cite{jin2020provably}. In contrast, we allow for arbitrary initial states, as in the regret analyses from \cite{jin2020provably,zanette2020frequentist}.

{\bf Low inherent Bellman error:} \cite{zanette2020learning} proposed the low inherent Bellman error (LIBE) model, which generalizes MLMDPs while retaining a linear flavor. For this model, \cite{zanette2020learning} improved the regret of \cite{jin2020provably}, while \cite{hu2021near} studied multi-task learning. However, these algorithms are defined in terms of optimization problems but no solutions are provided, so the algorithms lack explicit implementation. \cite{zanette2020provably} proved a sample complexity bound, which yields a regret minimization algorithm that satisfies \ref{itComp} but requires i.i.d.\ initial states. All violate \ref{itPar}.

{\bf Linear mixture:} The linear mixture model (LMM) assumes the transition kernel is a linear combination of $d$ known measures (see, e.g., \cite{jia2020model,modi2020sample,zhang2021variance,zhou2021nearly,zhou2021provably}), which is distinct from our Assumption \ref{assLin}. To our knowledge, the only regret bound for misspecified LMMs is from \cite{ayoub2020model}; the algorithm satisfies \ref{itRegret} but violates \ref{itComp} and \ref{itPar}.

{\bf Nonlinear generalizations:} Some nonlinear generalizations of LMDPs have been proposed, such as the case where the state-action value function belongs to a class of bounded eluder dimension \citep{russo2013eluder} or can be represented by a kernel function or neural network. While such generalization is important, these works (see, e.g., \cite{chowdhury2020no,ishfaq2021randomized,kong2021online,wang2020reinforcement,wang2020optimism,yang2020provably,yang2020function}) fail to improve over \cite{jin2020provably,zanette2020frequentist} in terms of \ref{itRegret}, \ref{itComp}, or \ref{itPar} (or regret).

{\bf Model selection:} To overcome the fact that existing MLMDP algorithms require $\epsmis$ as input, one could alternatively use a model selection algorithm (see, e.g., \cite{cutkosky2021dynamic,pacchiano2020regret,pacchiano2020model}). In our context, these initialize $M$ base algorithms (e.g., \texttt{LSVI-UCB}) with respective inputs $\epsmis(1), \ldots, \epsmis(M)$. Then at each episode, the bases compute policies and a master algorithm uses past data to choose one of the policies. To our knowledge, the only explicit results use \cite{zanette2020learning} as the base, which lacks implementation, and while the resulting master achieves \ref{itPar}, it violates \ref{itRegret} and \ref{itComp} (see Appendix D.4 in \cite{cutkosky2021dynamic} and Section 6.4 in \cite{pacchiano2020regret}). Another downside is that the master requires a regret bound for each base, so while \ref{itPar} holds, the stronger ``parameter free" property that neither $\epsmis$ nor regret bounds are known (which we satisfy) is violated.

{\bf MLCBs:} In the special case of MLCBs ($H=1$), \cite{gopalan2016low} showed (unmodified) \texttt{Lin-UCB} can achieve sublinear regret when $\epsmis$ is very small. \cite{foster2020beyond} proved regret bounds more generally but \ref{itPar} fails. \cite{foster2021adapting} provided expected regret bounds for an algorithm that satisfies \ref{itPar}. As mentioned above, \cite{takemura2021parameter}'s algorithm satisfies \ref{itPar}, and their bounds hold with high probability. In the noncontextual case, \cite{lattimore2020learning} proposed an algorithm that achieves \ref{itRegret} or \ref{itPar}, but not both. \cite{ghosh2017misspecified} proved regret bounds that may be polynomial in $|\A|$.

{\bf Other related work:} \cite{dong2019provably} considered a misspecified state aggregation model; their algorithm satisfies \ref{itRegret} and \ref{itComp} but not \ref{itPar}. \cite{lattimore2020learning} proved sample complexity bounds for discounted MDPs where the $Q$-function is approximately linear (more general than us), but they require a simulator/generative model. \cite{yin2021efficient} considered a similar setting, though only requires ``local'' simulator access. \cite{wang2021sample} only assumed the transition kernel is linear but requires a simulator.

\section{Preliminaries} \label{secPrelim}

{\bf Finite-horizon MDP:} We use the standard notation. $\S$ is the state space, $\A$ is the finite action space, $H \in \N$ is the horizon, $\{ r_h \}_{h=1}^H$ are the mean rewards, and $\{ P_h \}_{h=1}^{H-1}$ are the transition kernels. We assume $r_h(s,a) \in [0,1]$ for each $h \in [H] = \{1,\ldots,H\}$, $s \in \S$, and $a \in \A$. We let $\Pi$ denote the set of policies, i.e., the set of sequences $\pi = ( \pi_h )_{h=1}^H$ with $\pi_h : \S \rightarrow \A$ for each $h$. For any $\pi \in \Pi$ and $h \in [H]$, $V_h^\pi : \S \rightarrow [H-h+1]$ denotes the value function
\begin{equation}
V_h^\pi(s) = \E \left[ \sum_{h'=h}^H r_{h'} ( s_{h'} , \pi_{h'}(s_{h'} ) ) \middle| s_h = s \right]  ,
\end{equation}
where $s_{h'+1} \sim P_{h'}(\cdot|s_{h'},\pi_{h'}(s_{h'}))$ for each $h'$. We let $Q_h^\pi : \S \times \A \rightarrow [H-h+1]$ denote the state-action value function (or $Q$-function) given by
\begin{equation} \label{eqQfunction}
Q_h^\pi(s,a) = r_h(s,a) + \E [ V_{h+1}^\pi(s_{h+1}) | s_h = s , a_h = a ] ,
\end{equation}
where $s_{h+1} \sim P_h(\cdot|s,a)$ and $V_{H+1}^\pi(s_{H+1}) = 0$ by convention. It is well known that there exists an optimal policy $\pi^\star = ( \pi_h^\star )_{h=1}^H$, i.e., $V_h^\star(s) \triangleq V_h^{\pi^\star}(s) = \max_{ \pi \in \Pi } V_h^\pi(s)$ for all $h \in [H]$ and $s \in \S$. Also, for any $h \in [H]$, $\pi_h^\star$ is greedy with respect to $Q_h^\star \triangleq Q_h^{\pi^\star}$, i.e., $V_h^\star(s) = \max_{a \in \A} Q_h^\star(s,a)$ for each $s \in \S$.

{\bf MLMDP:} As discussed in Section \ref{secIntro}, we make the following linearity assumption.
\begin{ass}[MLMDP] \label{assLin}
For some known $\phi : \S \times \A \rightarrow \R^d$ and all $h \in [H]$, there exists unknown $\theta_h \in \R^d$ and $d$  unknown measures $\mu_h = ( \mu_{h,1} , \ldots , \mu_{h,d} )$ over $\S$ such that, for all $s \in \S$ and $a \in \A$,
\begin{gather}
| r_h(s,a) - \phi(s,a)^\trans \theta_h | \leq \epsmis , \\
 \| P_h(\cdot|s,a) - \phi(s,a)^\trans \mu_h \|_1  \leq \epsmis  .
\end{gather}
We also have $\max_{(s,a) \in \S \times \A} \|\phi(s,a)\|_2 \leq 1$ and $\max_{h \in [H]} \max \{ \| \theta_h \|_2 , \int_{s' \in \S} \| \mu_h(s') \|_2 \} \leq \sqrt{d}$.
\end{ass}

\begin{rem}[Comparison to prior work]
Assumption \ref{assLin} matches Assumption B of \cite{jin2020provably} and Assumption 1 of \cite{zanette2020frequentist}, except the latter has general $\ell_2$ norm bounds (e.g., $\| \phi(s,a) \|_2 \leq L_\phi$ for some $L_\phi$). Our analysis can be similarly generalized.
\end{rem}

A key consequence is that the $Q$-function is approximately linear. (See Proposition 2.3 of \cite{jin2020provably} or Corollary B.3 of \cite{zanette2020frequentist} for a proof.)
\begin{prop}[MLMDP $Q$-function] \label{propQfunction}
For any $\pi \in \Pi$ and $h \in [H]$, there exists $w_h^\pi \in \R^d$ such that $|Q_h^\pi(s,a) - \phi(s,a)^\trans w_h^\pi| \leq (H-h+1) \epsmis$ for any $(s,a) \in \S \times \A$.
\end{prop}

{\bf Regret:} We follow the standard episodic framework. At episode $k$, we choose policy $\pi^k = ( \pi_h^k )_{h=1}^H$ and begin at an arbitrary initial state $s_1^k \in \S$. For each $h \in [H]$, we take action $a_h^k = \pi_h^k(s_h^k)$, observe noisy reward $r_h^k(s_h^k,a_h^k) = r_h(s_h^k,a_h^k) + \eta_h^k$ (where $\eta_h^k$ is conditionally zero-mean noise, i.e., $\E [ \eta_h^k | s_h^k , a_h^k ] = 0$), and (when $h < H$) transition to $s_{h+1}^k \sim P_h(s_h^k,a_h^k)$. We assume $r_h(s_h^k,a_h^k)$ and $r_h^k(s_h^k,a_h^k)$ lie in $[0,1]$, so $\eta_h^k \in [-1,1]$.\footnote{The results extend with minor modification to bounded mean rewards and subgaussian noise.} We measure performance in terms of the regret
\begin{equation} \label{eqRegret}
R(K) = \sum_{k=1}^K \left( V_1^\star(s_1^k) - V_1^{\pi_k}(s_1^k) \right) .
\end{equation}

{\bf MLCB:} For MLCBs, we use the notation above with $H = 1$ and discard subscripts $h$. So, for each $k \in [K]$, we observe context $s^k \in \S$, take action $a^k \in \A$, and receive reward $r^k(s^k,a^k) = r(s^k,a^k) + \eta^k$. As above, $\eta^k$ is conditionally zero-mean, $r^k(s^k,a^k)$ and $r(s^k,a^k)$ are $[0,1]$-valued, and $| r(s,a) - \phi(s,a)^\trans \theta | \leq \epsmis$.

\section{MLCB algorithms} \label{secAlgCB}

\IncMargin{1em}
\begin{algorithm}[t] \label{algExpl3}

\caption{$\texttt{EXPL3}(thres)$}

$\Psi^0 = \emptyset$

\For{episode $k = 1, \ldots , K$}{

Observe $s^k \in \S$ 

$\Lambda^k = I + \sum_{ \tau \in \Psi^{k-1} } \phi(s^\tau,a^\tau) \phi(s^\tau,a^\tau)^\trans$

$w^k = ( \Lambda^k )^{-1} \sum_{ \tau \in \Psi^{k-1} } \phi(s^\tau,a^\tau) r^\tau(s^\tau,a^\tau)$

\If{$\max_{a \in \A} \|\phi(s^k,a) \|_{(\Lambda^k)^{-1}} > thres$}{ \label{lnExpl3check}

$a^k = \argmax_{a \in \A} \| \phi(s^k,a) \|_{(\Lambda^k)^{-1}}$ \label{lnExpl3explore}

$\Psi^k = \Psi^{k-1} \cup \{k\}$ 
}
\Else{

$a^k = \argmax_{a \in \A} \phi(s^k,a)^\trans w^k$ , $\Psi^k = \Psi^{k-1}$  \label{lnExpl3exploit}

}

Play $a^k$, observe $r^k(s^k,a^k)$

}

\end{algorithm}
\DecMargin{1em}

In this section, we restrict to MLCBs and discuss \texttt{EXPL3} and our \texttt{Sup-Lin-UCB} variant. We will later leverage the insights developed in this section to design our MLMDP algorithm (see, e.g., Remark \ref{remOffPolicy}).

{\bf Warm-up: noncontextual, known $\epsmis$:} Assume momentarily that $s^1 = \cdots = s^K$ and $\epsmis$ is known. In this case, we propose a simple two stage algorithm. For the first stage, we play the action with the highest noise/uncertainty in order to reduce it. Mathematically, we quantify the uncertainty in the usual way for LCBs: let $\Lambda^k = I + \sum_{\tau \leq k} \phi(s^\tau,a^\tau) \phi(s^\tau,a^\tau)^\trans$,
\begin{equation} \label{eqInducedNorm}
\|\phi(s^k,a) \|_{(\Lambda^k)^{-1}} = \sqrt{ \phi(s^k,a)^\trans (\Lambda^k)^{-1} \phi(s^k,a) } ,
\end{equation}
and play $a^k = \argmax_a \|\phi(s^k,a) \|_{(\Lambda^k)^{-1}}$. After $\approx \epsmis^{-2}$ such episodes, the uncertainty falls below $\epsmis$, which means $|\phi(s^k,a)^\trans (w^k - \theta) | = O ( \epsmis )$ for the least-squares estimate $w^k$ of $\theta$. 
Hence, $\argmax_a\phi(s^k,a)^\trans w^k$ is an $\epsmis$-suboptimal policy, which is the best possible under Assumption \ref{assLin}. Accordingly, the second stage exploits by playing $a^k = \argmax_a\phi(s^k,a)^\trans w^k$.

{\bf General case:} For the contextual setting and unknown $\epsmis$, \texttt{EXPL3} (Algorithm \ref{algExpl3}) generalizes this approach. Note uncertainty now depends on the context $s^k$, and we can no longer define ``high'' as ``$\geq \epsmis$''. Thus, at episode $k$, \texttt{EXPL3} checks if the uncertainty at the current context $s^k$ exceeds the input $thres$ (Line \ref{lnExpl3check}). If so, it explores a high uncertainty action (Line \ref{lnExpl3explore}); otherwise, it exploits the estimated best action (Line \ref{lnExpl3exploit}). Here $\Lambda^k$ and $w^k$ are only computed from episodes $\Psi^{k-1}$ where \texttt{EXPL3} explored. When $thres \approx \epsmis$, $\texttt{EXPL3}(thres)$ is essentially a contextual version of the previous approach and should perform well. However, if $thres \not \approx \epsmis$, it will fail for one of two reasons:
\begin{enumerate}[leftmargin=*,align=left,itemsep=0pt,topsep=0pt,label=(\Alph*)]
\item \label{itThresSubOpt} If $thres \gg \epsmis$, $\texttt{EXPL3}(thres)$ stops exploring as soon as the uncertainty falls below $thres$, so it only learns a $thres$-suboptimal policy. 
\item \label{itExtraReg} If $thres \ll \epsmis$, $\texttt{EXPL3}(thres)$ explores too much -- for roughly $thres^{-2}$ episodes when $\epsmis^{-2}$ suffice. This may cause $\text{sp}(r) thres^{-2}$ additional regret, where $\text{sp}(r) = \max_a r(s^k,a) - \min_a r(s^k,a)$. 
\end{enumerate}

{\bf Ensemble approach:} We next show these failures can be overcome by carefully interconnecting the ensemble $\{ \texttt{EXPL3}(2^{-l}) \}_{l=1}^L$. Roughly, the $l$-th will explore until its uncertainty is $2^{-l}$, as above. Thereafter, we know its greedy policy is only $2^{-l}$-suboptimal, and since $\epsmis$ is unknown, we may have $2^{-l} \gg \epsmis$. Hence, instead of choosing the best action -- which causes failure \ref{itThresSubOpt} -- \textit{we only ask it to eliminate $2^{-l}$-suboptimal actions}. By the same token, the $(l-1)$-th algorithm has already eliminated $2^{1-l}$-suboptimal actions, so \textit{the $l$-th sees a more favorable problem instance}, with $\text{sp}(r) = O( 2^{-l} )$. Thus, if instead $2^{-l} \ll \epsmis$, the extra regret in \ref{itExtraReg} scales as $2^{-l} ( 2^{-l} )^{-2} = 2^l$, which is tolerable if we choose $L = \log_2(O(\sqrt{K}))$. 

More precisely, our \texttt{Sup-Lin-UCB} variant is given in Algorithm \ref{algSLU}. At episode $k$, it chooses an action via a phased elimination procedure that lasts at most $L$ phases. Generalizing \texttt{EXPL3}, it computes $\Lambda_l^k$ and $w_l^k$ using $\Psi_l^k$, which are the exploratory episodes at phase $l$. As discussed above, the $l$-th phase chooses a high uncertainty action if one exists (Line \ref{lnSLUexplore}) and otherwise eliminates actions with estimated reward $\Theta ( 2^{-l} )$ less than the maximal (Line \ref{lnSLUeliminate}). Finally, if phase $L$ is reached and an action was never chosen in Line \ref{lnSLUexplore}, it exploits the estimated best action (Line \ref{lnSLUexploit}).

\IncMargin{1em}
\begin{algorithm}[t] \label{algSLU}

\caption{\texttt{Sup-Lin-UCB-Var}}

$\Psi_l^0 = \emptyset\ \forall\ l \in [L]$

\For{episode $k = 1, \ldots , K$}{

Observe $s^k \in \S$, set $\A_1^k = \A$

\For{phase $l=1,\ldots,L$}{

$\Lambda_l^k = I + \sum_{ \tau \in \Psi_l^{k-1} } \phi(s^\tau,a^\tau) \phi(s^\tau,a^\tau)^\trans$

$w_l^k = ( \Lambda_l^k )^{-1} \sum_{ \tau \in \Psi_l^{k-1} } \phi(s^\tau,a^\tau) r^\tau(s^\tau,a^\tau)$

\uIf{$\max_{a \in \A_l^k} \|\phi(s^k,a) \|_{(\Lambda_l^k)^{-1}} > 2^{-l}$}{ 

$a^k = \argmax_{a \in \A_l^k} \| \phi(s^k,a) \|_{(\Lambda_l^k)^{-1}}$ \label{lnSLUexplore} 

$\Psi_l^k = \Psi_l^{k-1} \cup \{k\}$, $\Psi_{l'}^k = \Psi_{l'}^{k-1}\ \forall\ l' \neq l$

\Break 

}
\uElseIf{$l < L$}{

$\A_{l+1}^k = \{ a \in \A_l^k : \phi(s^k,a)^\trans w_l^k  \geq \max_{a' \in \A_l^k} \phi(s^k,a')^\trans w_l^k - \Theta ( 2^{-l} ) \}$ 
\label{lnSLUeliminate}

}
\uElse{

$a^k = \argmax_{a \in \A_L^k} \phi(s^k,a)^\trans w_L^k$ \label{lnSLUexploit}

$\Psi_{l'}^k = \Psi_{l'}^{k-1}\ \forall\ l'$  

}

}

Play $a^k$, observe $r^k(s^k,a^k)$

}

\end{algorithm}
\DecMargin{1em}

{\bf Formal interpretation:} To complement this intuition, we provide a formal result. Roughly, it shows that if $\S$ is rich enough, then for any phase $l$ and any contexts in \texttt{Sup-Lin-UCB-Var}, there are contexts in $\texttt{EXPL3}(2^{-l})$ such that the latter learns the same policy as the $l$-th phase of \texttt{Sup-Lin-UCB-Var}. In other words, \texttt{Sup-Lin-UCB-Var} runs $\{ \texttt{EXPL}(2^{-l}) \}_{l=1}^L$. In this way, \texttt{Sup-Lin-UCB} is akin to model selection, but unlike those approaches, does not attempt to learn the best $\texttt{EXPL}(2^{-l})$. See Appendix \ref{appEnsembleProof} for a proof.
\begin{prop}[\texttt{Sup-Lin-UCB} $=$ \texttt{EXPL3} ensemble] \label{propEnsemble}
Assume that for any $\{ \phi_a \}_{a \in \A} \subset \R^d$ and $\{ r_a \}_{a \in \A} \subset \R$, there exists $s \in \S$ such that $\phi(s,a) = \phi_a$ and $r(s,a) = r_a\ \forall\ a \in \A$. Then for any $l \in [L]$ and $\{ s^k \}_{k=1}^K \subset \S$, there exists $\{ \tilde{s}^k \}_{k=1}^K \subset \S$ such that, if Algorithms \ref{algExpl3} and \ref{algSLU} are run with contexts $\{ \tilde{s}^k \}_{k=1}^K$ and $\{ s^k \}_{k=1}^K$, respectively, and if both face the same noise sequence $\{ \eta^k \}_{k=1}^K$, then $(w^k , \Lambda^k) = ( w_l^k, \Lambda_l^k )\ \forall\ k \in [K]$. 
\end{prop}

\begin{rem}[Historical note]
\texttt{Sup-Lin-UCB-Var} simplifies \cite{takemura2021parameter}'s algorithm (see Appendix \ref{appCompareSLU}), which modifies \texttt{Sup-Lin-UCB} \citep{chu2011contextual}, which builds upon \texttt{Sup-Lin-Rel} \citep{auer2002using}. The latter three set $L = \log_2(O(\sqrt{K}))$ as above, but we keep it general, which is crucial for \ref{itRegret} and \ref{itComp}.
\end{rem}

\section{MLMDP algorithm} \label{secAlgMDP}

We can now leverage the intuition of the previous section to discuss \texttt{Sup-LSVI-UCB} (Algorithm \ref{algMain}). To begin, it initializes the aforementioned parameter $L$, an exploration parameter $\alpha$, a rounding parameter $\epsrnd$ (to be discussed shortly), and the sets $\Psi_{h,l} = \emptyset$ (now indexed by step $h \in [H]$ but similar to $\Psi_l$ in Algorithm \ref{algSLU}). The $k$-th episodes then contains two parts:
\begin{itemize}[leftmargin=*,align=left,itemsep=0pt,topsep=0pt]

\item {\bf Policy update (Alg.\ \ref{algMain}, Lines \ref{lnStartInduct}-\ref{lnCallSub})}: Starting at $h = H$ and inducting backward, for each $l \in [L]$, Lines \ref{lnLambdaUR} and \ref{lnWeightVectorUR} use the episodes $\Psi_{h,l}^k$ to compute a least-squares estimate $w_{h,l}^k$ of the vector $w_h^\pi$ from Proposition \ref{propQfunction} (and the matrix $\Lambda_{h,l}^k$). This is the same approach used by \texttt{LSVI-UCB}; see Section 4 of \cite{jin2020provably} for intuition. The difference is that the next-state value estimate $V_{h+1}^k(s_{h+1}^\tau)$ in Line \ref{lnWeightVectorUR} is computed via Algorithm \ref{algSub}, which is essentially \texttt{Sup-Lin-UCB-Var} and will be discussed soon. In contrast, \texttt{LSVI-UCB} uses \texttt{Lin-UCB}-style estimates
\begin{align} \label{eqJinEstimate}
\max_{a \in \A}  \left( \phi(s,a)^\trans w_{h+1}^k + \alpha \| \phi(s,a) \|_{ (\Lambda_{h+1}^k )^{-1} }  \right) ,
\end{align}
where $w_{h+1}^k$ and $\Lambda_{h+1}^k$ are computed using all data (not just $\Psi_{h,l}^k$). Additionally, Lines \ref{lnLambdaR} and \ref{lnWeightVectorR} elementwise round $w_{h,l}^k$ and $( \Lambda_{h,l}^k )^{-1}$ to $\tilde{w}_{h,l}^k$ and $( \tilde{\Lambda}_{h,l}^k )^{-1}$ for reasons discussed in Remark \ref{remRounding}.\footnote{We emphasize $(\tilde{\Lambda}_{h,l}^k )^{-1}$ elementwise rounds $( \Lambda_{h,l}^k )^{-1}$; we have \textit{not} defined $\tilde{\Lambda}_{h,l}^k$ and inverted it.}

\item {\bf Policy execution (Alg.\ \ref{algMain}, Lines \ref{lnObserveS1}-\ref{lnPsiCase2}):} After updating the policy, we execute it by computing its relevant entries $\{ \pi_h^k(s_h^k) \}_{h=1}^H$ via Algorithm \ref{algSub}. (We discuss the $\Psi_{h,l}$ update in Lines \ref{lnPsiCase1A}, \ref{lnPsiCase1B}, and \ref{lnPsiCase2} soon.)
\end{itemize}

\IncMargin{1em}
\begin{algorithm}[t] \label{algMain}

\caption{$\texttt{Sup-LSVI-UCB}(\epstol,\delta)$}

$L = \ceil{ \log_2  ( \sqrt{d}/\epstol) }$, $\alpha = 42 d H L \sqrt{ \log( \frac{3 dHL}{\delta})}$

$\epsrnd = 2^{-4L}/d$, $\Psi_{h,l}^0 = \emptyset\ \forall\ l \in [L] , h \in [H]$

\For{episode $k = 1 , \ldots , K$}{

\For{step $h = H , \ldots , 1$}{ \label{lnStartInduct}

\For{phase $l = 1 , \ldots , L$}{

$\Lambda_{h,l}^k = 16 I + \sum_{ \tau \in \Psi_{h,l}^{k-1} } \phi(s_h^\tau,a_h^\tau) \phi(s_h^\tau,a_h^\tau)^\trans$ \label{lnLambdaUR}

$w_{h,l}^k= ( \Lambda_{h,l}^k )^{-1} \sum_{ \tau \in \Psi_{h,l}^{k-1} } \phi(s_h^\tau,a_h^\tau) \times ( r_h^\tau(s_h^\tau,a_h^\tau) + V_{h+1}^k ( s_{h+1}^\tau ) )$ \label{lnWeightVectorUR}

$( \tilde{\Lambda}_{h,l}^k )^{-1} = \epsrnd \ceil{ ( \Lambda_{h,l}^k )^{-1}  / \epsrnd }$ \label{lnLambdaR}

$\tilde{w}_{h,l}^k = \epsrnd \ceil{ w_{h,l}^k / \epsrnd }$  \label{lnWeightVectorR}

}

$( \pi_h^k(\cdot) , V_h^k(\cdot) , l_h^k(\cdot) ) = \texttt{Sup-Lin-UCB-Var}(\cdot)$ (see Algorithm \ref{algSub}) \label{lnCallSub}

}

Observe $s_1^k \in \S$ \label{lnObserveS1}

\For{step $h = 1 , \ldots , H$}{

Play $a_h^k = \pi_h^k(s_h^k)$, observe $r_h^k(s_h^k,a_h^k)$ \label{lnTakeAction}
 
Transition to $s_{h+1}^k \sim P_h(\cdot|s_h^k,a_h^k)$  

\uIf{$l_h^k(s_h^k) \leq L$}{

$\Psi_{h,{l_h^k}(s_h^k)}^k = \Psi_{h,{l_h^k}(s_h^k)}^{k-1} \cup \{k\}$ \label{lnPsiCase1A}

$\Psi_{h,l}^k = \Psi_{h,l}^{k-1}\ \forall\ l \in [L] \setminus \{ l_h^k(s_h^k) \}$ \label{lnPsiCase1B}

}
\lElse{$\Psi_{h,l}^k = \Psi_{h,l}^{k-1}\ \forall\ l \in [L]$ \label{lnPsiCase2}}

}

}

\end{algorithm}
\DecMargin{1em}

\IncMargin{1em}
\begin{algorithm}[t] \label{algSub}

\caption{$\texttt{Sup-Lin-UCB-Var}(s)$}

$\A_{h,1}^k(s) = \A$, $V_{h,0}^k(s) = H$

\For{phase $l=1,\ldots,L$}{

\uIf{$\max_{ a \in \A_{h,l}^k(s) } \| \phi(s,a) \|_{ (\tilde{\Lambda}_{h,l}^k )^{-1} } > 2^{-l}$}{ \label{lnSubExploreCond}

$\pi_h^k(s) = \argmax_{a \in \A_{h,l}^k(s)} \| \phi(s,a) \|_{ (\tilde{\Lambda}_{h,l}^k )^{-1} }$  \label{lnSubExplore}

$V_h^k(s) = \mathcal{P}_{[0,H]} ( V_{h,l-1}^k(s) + 2^{1-l} \alpha )$, $l_h^k(s) = l$\label{lnSubExploreOther}

\Return{$(\pi_h^k(s),V_h^k(s),l_h^k(s))$} \label{lnSubExploreReturn}

}
\uElseIf{$l < L$}{ 

$\pi_{h,l}^k(s) = \argmax_{ a \in \A_{h,l}^k(s) } \phi(s,a)^\trans \tilde{w}_{h,l}^k$ \label{lnSubEstBest}

$V_{h,l}^k(s) = \phi(s,\pi_{h,l}^k(s))^\trans \tilde{w}_{h,l}^k$ \label{lnSubEstVal}

$\A_{h,l+1}^k(s) = \{ a \in \A_{h,l}^k(s) : \phi(s,a)^\trans \tilde{w}_{h,l}^k \geq V_{h,l}^k(s) -  2^{1-l} \alpha \}$ \label{lnSubEliminate}

}
\uElse{

$\pi_h^k(s) = \pi_{h,L}^k(s)$ \label{lnSubExploit} 

$V_h^k(s) = \mathcal{P}_{[0,H]} ( V_{h,L}^k(s) )$, $l_h^k(s) = L+1$ \label{lnSubExploitOther}

\Return{$(\pi_h^k(s),V_h^k(s),l_h^k(s))$}

}
}
\end{algorithm}
\DecMargin{1em}

{\bf Subroutine (Alg.\ \ref{algSub}):} As mentioned above, Algorithm \ref{algSub} implements \texttt{Sup-Lin-UCB-Var} logic to choose $\pi_h^k(s)$ for a given $s \in \S$, though using the rounded $\tilde{w}_{h,l}^k$ and $( \tilde{\Lambda}_{h,l}^k )^{-1}$. A small technical issue is that the analogue of \eqref{eqInducedNorm} may be ill-defined, so instead we let\footnote{While this need not be a norm, $\epsrnd$ is small enough that it behaves like one (at least enough for our purposes).}
\begin{equation}
\| \phi(s,a) \|_{ ( \tilde{\Lambda}_{h,l}^k )^{-1} } = \sqrt{ | \phi(s,a)^\trans ( \tilde{\Lambda}_{h,l}^k )^{-1} \phi(s,a) | } .
\end{equation}
Algorithm \ref{algSub} also returns $V_h^k(s)$, which is the value estimate used in the backward induction, and $l_h^k(s)$, where $l_h^k(s)-1$ is the number of eliminations conducted (see Lines \ref{lnSubExploreOther} and \ref{lnSubExploitOther}). For the value estimate, we use $\mathcal{P}_{[0,H]}$ to project onto $[0,H]$, i.e., $\mathcal{P}_{[0,H]}(x) = 0$, $x$, and $H$ when $x < 0$, $\in [0,H]$, and $> H$, respectively. This is typical for MLMDPs and ensures boundedness of the random variables. See Remark \ref{remOffPolicy} for further discussion of $V_h^k$. Also note $l_h^k(s) \leq L$ only when an exploratory action is chosen in Line \ref{lnSubExplore}, in which case \texttt{Sup-LSVI-UCB} adds the episode to $\Psi_{h,l}$ (Line \ref{lnPsiCase1A} of Algorithm \ref{algMain}).

\begin{rem}[Rounding] \label{remRounding}
The rounding in Lines \ref{lnLambdaR} and \ref{lnWeightVectorR} of Algorithm \ref{algMain} ensures that $V_{h+1}^k : \S \rightarrow [0,H]$ defined by Algorithm \ref{algSub} belongs to a finite function class. This enables a union bound over the function class in our concentration lemma, which is needed because $V_{h+1}^k$ is a random function that depends on past data. In contrast, \cite{jin2020provably} shows \eqref{eqJinEstimate} is close to a function class with a small covering number, then takes a union bound over the cover. This relies on the fact that their $\S$-dimensional value function estimate is itself a continuous function of the $\text{poly}(d)$-dimensional $w_h^k$ and $\Lambda_h^k$. In our case,  Lines \ref{lnSubExploreCond} and \ref{lnSubEliminate} of Algorithm \ref{algSub} introduce discontinuities that cause this to fail.
\end{rem}

\begin{rem}[Off-policy estimates] \label{remOffPolicy}
Algorithm \ref{algSub}'s (unprojected) value estimate takes one of two forms. In Line \ref{lnSubExploitOther}, it is $\phi(s,\pi_h^k(s))^\trans \tilde{w}_{h,L}^k$, which is the $Q$-function estimate at the chosen action $\pi_h^k(s)$. In this sense, it is the usual ``on-policy'' estimate used in \texttt{LSVI-UCB} and most other algorithms. In contrast, Line \ref{lnSubExploreOther} uses
\begin{equation} \label{eqOffPolicy}
\phi(s,\pi_{h,l-1}^k(s))^\trans \tilde{w}_{h,l-1}^k  + \Theta ( 2^{-l} ) .
\end{equation}
Since $\pi_{h,l-1}^k(s) \neq \pi_h^k(s)$ in general, this estimate is off-policy. We use such estimates in Line \ref{lnSubExploreOther} because Lines \ref{lnSubExplore}-\ref{lnSubExploreReturn} correspond to the explicit exploration discussed in Section \ref{secAlgCB}. When such exploration is needed, the on-policy estimate may be quite low, so we use an off-policy estimate to incentivize the algorithm to visit $s$ (after which we can conduct the exploration). The $2^{-l}$ term in \eqref{eqOffPolicy} is also motivated by Section \ref{secAlgCB}. In essence, since $\texttt{EXPL}(2^{-l})$ learns up to $2^{-l}$ noise, it is an uncertainty bonus that makes \eqref{eqOffPolicy} the highest statistically-plausible value from the perspective of $\texttt{EXPL}(2^{-l})$.
\end{rem}

\begin{rem}[Adversarial corruptions]
\cite{lykouris2021corruption} assume the MDP is linear ($\epsmis = 0$) except for a small number of episodes, where it changes arbitrarily. While quite different from MLMDP, they also combine backward induction and phased elimination, with each phase using a subset of episodes. The commonality is high level, though; in that work, each phase runs its own backward induction subroutine, which episodes each phase uses differ, and the algorithm is model-based (among other differences).
\end{rem}

\section{Main results} \label{secResults}

We can now present Theorem \ref{thmMain}, which provides regret and complexity guarantees for \texttt{Sup-LSVI-UCB}. This is our most general result, and we will soon examine some special cases of $\epstol$ to build further intuition. For now, we mention that the theorem (and the fact that $\epsmis$ does not appear in our algorithms) guarantees \ref{itRegret}, \ref{itComp}, and \ref{itPar} hold when $\epstol$ is independent of $K$, which (to our knowledge) is a first for MLMDPs.

\begin{thm}[General result]\label{thmMain}
If Assumption \ref{assLin} holds and we run Algorithm \ref{algMain} with inputs $\epstol \in (0,1)$ and $\delta \in (0,1)$, then with probability at least $1-\delta$,
\begin{align}
R(K) & = O \Big( \sqrt{ d^3 H^4  \min \{ (d / \epstol )^2 , K \}  \log^5 ( \tfrac{d}{\epstol} ) \iota } \\
& \qquad + \sqrt{H^3 K \iota} \\
& \qquad +  \sqrt{ d } H^2  K \max \{ \epsmis , \epstol \}  \sqrt{  \log^3 ( \tfrac{d}{\epstol} ) \iota }  \Big) ,
\end{align}
where $\iota = \log(3dH/\delta)$. Furthermore, Algorithm \ref{algMain}'s space complexity is
\begin{equation} \label{eqThmSpace}
O \left( d^2 H \log(\tfrac{d}{\epstol}) + d H |\A| \min \left\{ d^2 \log ( \tfrac{d}{\epstol} ) / \epstol^2 , K \right\} \right) ,
\end{equation}
and its per-episode time complexity is
\begin{equation} \label{eqThmTime}
O \left( d^2 H |\A|  \min \left\{ d^2 \log( \tfrac{d}{\epstol} ) / \epstol^2 , K \right\} \log ( \tfrac{d}{\epstol} ) \right) .
\end{equation}
\end{thm}

Alternatively, if $K$ is small (so \ref{itRegret} and \ref{itComp} are less relevant), we can choose $\epstol$ in terms of $K$ to obtain the following corollary. Here \ref{itPar} still holds, though (like many existing algorithms) \ref{itRegret} and \ref{itComp} fail. 
\begin{cor}[Unknown $\epsmis$, small $K$] \label{corUnknown}
If Assumption \ref{assLin} holds and we run Algorithm \ref{algMain} with inputs $\epstol = \frac{d}{\sqrt{K}}$ and $\delta \in (0,1)$, then with probability at least $1-\delta$,
\begin{align}
R(K) & = O \Big( \sqrt{ d^3 H^4 K \log^5(K) \iota }  \\
& \qquad +  \sqrt{ d } H^2 K \epsmis \sqrt{ \log^3(K) \iota }  \Big)  ,
\end{align}
where $\iota = \log(3dH/\delta)$. Furthermore, Algorithm \ref{algMain}'s space complexity is $O ( d^2 H \log(K) + d H K |\A| )$ and its per-episode time complexity is $O ( d^2 H K |\A| \log (K) )$.
\end{cor}

Finally, given knowledge of $\epsmis$ as in \cite{jin2020provably,zanette2020frequentist}, we can set $\epsmis = \epstol$ to ensure \ref{itRegret} and \ref{itComp} hold with the following regret bound.
\begin{cor}[Known $\epsmis$, large $K$] \label{corKnown}
If Assumption \ref{assLin} holds and we run Algorithm \ref{algMain} with inputs $\epstol = \epsmis$ and $\delta \in (0,1)$, then with probability at least $1-\delta$,
\begin{align}
R(K) & = O \Big( \sqrt{ d^3 H^4  \min \{ (d / \epsmis )^2 , K \}  \log^5 ( \tfrac{d}{\epsmis} ) \iota } \\
& \qquad + \sqrt{H^3 K \iota} +  \sqrt{ d } H^2 K \epsmis \sqrt{  \log^3 ( \tfrac{d}{\epsmis} ) \iota }  \Big) ,
\end{align}
where $\iota = \log(3dH/\delta)$. Furthermore, Algorithm \ref{algMain}'s space and per-episode time complexities are bounded by \eqref{eqThmSpace} and \eqref{eqThmTime}, respectively, with $\epstol$ replaced by $\epsmis$.
\end{cor}

\begin{rem}[Comparison to prior work]
Up to log factors, the ``linear'' terms in the corollaries improve existing results by a $\sqrt{d}$ factor; the sublinear term in Corollary \ref{corUnknown} matches the best known, while Corollary \ref{corKnown} improves it for $K \geq (d / \epstol )^2$ (see Table \ref{tabSummary}). The complexity bounds in Corollary \ref{corUnknown} match \cite{jin2020provably,zanette2020frequentist}, and Corollary \ref{corKnown} improves these bounds (also for $K \geq (d / \epstol )^2$).
\end{rem}

\begin{rem}[Linear term] \label{remLinear}
Again neglecting logs, the linear terms in the corollaries are $\sqrt{d} H^2 K \epsmis$. \cite{lattimore2020learning} shows the $\sqrt{d}$ ``blow-up'' is unavoidable and conjectures that for $\gamma$-discounted MDPs, $(1-\gamma)^{-2}$ dependence (the analogue of $H^2$ in the discounted setting) is optimal.
\end{rem}

\begin{rem}[Choice of $L$]
When $\epstol = \frac{d}{\sqrt{K}}$ in Corollary \ref{corUnknown}, $L = \log_2(\frac{K}{\sqrt{d}})$ in Algorithm \ref{algSub}, which is the choice used in \cite{takemura2021parameter}. When $\epstol = \epsmis$ in Corollary \ref{corKnown}, $L = \log_2 ( \frac{\sqrt{d}}{\epsmis} )$, which roughly means the ensemble $\{ \texttt{EXPL3}(2^{-l}) \}_{l=1}^L$ explores until the lowest noise level matches the misspecification bias.
\end{rem}

\begin{rem}[MLCB]
When $H=1$ and $L = \log_2(\frac{K}{\sqrt{d}})$, Corollary \ref{corUnknown} can be sharpened to $d \sqrt{K} + \sqrt{d} K \epsmis$ (see Remark \ref{remMLCB} in Appendix \ref{appMainProof}), which matches \cite{takemura2021parameter}'s result when $|\A|$ is large.
\end{rem}

We prove Theorem \ref{thmMain} in Appendix \ref{appMainProof}. At a high level, the proof generalizes that of \cite{takemura2021parameter}'s Theorem 1 from MLCBs to MLMDPs so is structurally similar. However, there are some key differences and challenges that are unique to the MLMDP setting:
\begin{itemize}[leftmargin=*,align=left,itemsep=0pt,topsep=0pt]

\item As discussed in Remark \ref{remRounding}, we use rounding to cope with dependent noise that arises when $H > 1$; see Lemma \ref{lemConcentration} for details. The downside is that rounding introduces additional errors. However, we generalize the proof in such a way that we can define an ``effective misspecification'' that accounts for both rounding error and misspecification, and that plays the same role the latter alone plays in \cite{takemura2021parameter} (see discussion preceding Lemma \ref{lemEstError}).

\item As discussed in Remark \ref{remOffPolicy}, we need to ensure the backward induction uses optimistic estimates $V_h^k(s)$ despite the fact that the algorithm occasionally takes exploratory (non-optimistic) actions. We show in Lemmas \ref{lemEstErrorLearner} and \ref{lemOptimism} that our definition of $V_h^k(s)$ in Algorithm \ref{algSub} judiciously balances two competing forces: estimating the value of the algorithm's policy (i.e., ensuring $V_h^k \approx V_h^{\pi_k}$) and remaining optimistic with respect to the optimal policy (i.e., $V_h^\star \approx V_h^k$), which together imply low regret (i.e., $V_h^\star \approx V_h^{\pi_k}$). 

\item In general, we are more careful with log terms that \cite{takemura2021parameter} simply bound by $\log K$ (see, e.g., discussion before Claim \ref{clmPsiBound}), as this leads to super-linear bounds that our analysis avoids.

\item Along these lines, the proof shows that at each phase of Algorithm \ref{algSub}, the misspecification may cause all $\epsmis$-suboptimal actions to be eliminated (see Claim \ref{clmPhaseDecay}, which generalizes \cite{takemura2021parameter}'s Lemma 4). Hence, after $\Omega(L)$ phases, Algorithm \ref{algSub} may recommend $\Omega(L\epsmis)$-suboptimal actions, which leads to super-linear regret bounds when $L$ grows with $K$. This is why we need to choose $\epstol$ (and subsequently $L$) independent of $K$ in order to achieve \ref{itRegret} in Theorem \ref{thmMain} and Corollary \ref{corKnown}.

\item Finally, when $H=1$, \cite{takemura2021parameter} separately bounds regret when (1) $l_1^k(s_1^k) = 1$, (2) $l_1^k(s_1^k) \in \{2,\ldots,L\}$, and (3) $l_1^k(s_1^k) = L+1$ (though they do not use this notation). For general $H \in \N$, we have an entire sequence $\{ l_h^k(s_h^k) \}_{h=1}^H$, which renders this case-based analysis intractable. Instead, we streamline their approach by showing (1) never occurs for our parameter choices (see Corollary \ref{corExploratory}) and by introducing the $l_h^k$ notation to treat (2) and (3) in a more unified manner (see, e.g., Claim \ref{clmEstErrorLearner}).

\end{itemize}

\section{Other results} \label{secLinUCB}

Finally, we return to discuss \texttt{Lin-UCB}. Recall we assume $\eta^k$ is zero-mean, $r(s^k,a^k)$ and $r^k(s^k,a^k) = r(s^k,a^k)+\eta^k$ lie in $[0,1]$ (so $\eta^k \in [-1,1]$), $| r(s,a) - \phi(s,a)^\trans \theta | \leq \epsmis$, and $\| \phi(s,a) \|_2 \leq 1$; we strengthen the assumption $\| \theta \|_2 \leq \sqrt{d}$ to $\| \theta \|_2 \leq 1$ in this section.\footnote{To prove \eqref{eqLattimoreRegret}, \cite{lattimore2020learning} assumes $1$-subgaussian noise (see their Section 5), $\| \phi(s,a) \|_2 \leq 1$, and $|\phi(s,a)^\trans \theta| \leq 1$ (see their Appendix E), which is similar.} For this setting, the regret definition \eqref{eqRegret} simplifies to
\begin{equation}
R(K) = \sum_{k=1}^K ( r(s^k,a_\star^k) - r(s^k,a^k) ) ,
\end{equation}
where $a^k_\star = \argmax_{a \in \A} r(s^k,a)$ and $a^k$ is the chosen action. In our notation, \texttt{Lin-UCB} chooses
\begin{equation} \label{eqLinUcb}
a^k = \argmax_{a \in \A} \left( \phi(s^k,a)^\trans w^k + \alpha \| \phi(s^k,a) \|_{ ( \Lambda^k )^{-1} } \right) ,
\end{equation}
where $\Lambda^k = \lambda I + \sum_{\tau=1}^{k-1} \phi(s^\tau,a^\tau) \phi(s^\tau,a^\tau)^\trans$ and $w^k = ( \Lambda^k )^{-1} \sum_{\tau=1}^{k-1} \phi(s^\tau,a^\tau) r^\tau (s^\tau,a^\tau)$. \cite{lattimore2020learning} (building upon \cite{jin2020provably}) show that choosing $\delta \in (0,1)$, $\lambda = 1$, and $\alpha = O ( \sqrt{ d \log (K/\delta) } +  \sqrt{K} \epsmis )$ ensures that with probability at least $1-\delta$,\footnote{The lemma actually bounds $\E[R(K)]$ for a refined algorithm, but \eqref{eqLattimoreRegret} can be similarly proven for \eqref{eqLinUcb}.}
\begin{equation} \label{eqLattimoreRegret}
R(K) = O \left( d \sqrt{K} \log (K/\delta) + \sqrt{d} K \epsmis \sqrt{\log K} \right) 
\end{equation}
(see their Lemma E.1). While $\sqrt{d} K \epsmis  \sqrt{\log K}$ is optimal up to the log term, it violates \ref{itRegret}. We show this can be remedied (and, when $K \gg \epsmis^{-2}$, \eqref{eqLattimoreRegret} improved) by choosing a different regularizer $\lambda$.
\begin{prop} \label{propLinUcb}
Let $\delta \in (0,1)$, $\lambda = 1 +  K \epsmis^2$, and $\alpha = 1 + \sqrt{2d \log ( (\lambda+K)/(\lambda \delta) ) } + 2  \sqrt{K} \epsmis$. Under the assumptions of Section \ref{secLinUCB}, with probability at least $1-\delta$, \texttt{Lin-UCB} \eqref{eqLinUcb} satisfies
\begin{align} \label{eqOurLinUcbRegret}
R(K) & = O \Big( d \sqrt{K} \log( \min \{ K , \epsmis^{-2} \} /\delta) \\
& \qquad + \sqrt{d} K \epsmis \sqrt{ \log ( \min \{ K , \epsmis^{-2} \} ) } \Big) .
\end{align}
\end{prop}
\begin{proof}[Proof idea]
The proof is standard; see Appendix \ref{appLinUcbProof} for a sketch. The key step is the sum-of-bonuses calculation, which (ignoring lower order terms) shows
\begin{align}
& \alpha \sum_{k=1}^K \| \phi(s^k,a^k) \|_{(\Lambda^k)^{-1}} = O \left( \epsmis K \sqrt{ d \log ( K  / \lambda )  } \right) .
\end{align}
Hence, when $\lambda = O(1)$ and $\lambda = O( K \epsmis^2)$, respectively, we obtain a super-linear term like \eqref{eqLattimoreRegret} and a linear term like \eqref{eqOurLinUcbRegret}, respectively.
\end{proof}

\begin{rem}[Intuition] \label{remRegIntuition}
Consider the case $\S = \{1\}$, $\A = \{1,\ldots,d\}$, and $\phi(1,i) = e_i$ (the $i$-th standard basis vector). For $\lambda = O(1)$ and $\alpha = \Omega ( \sqrt{K} \epsmis )$ as in \cite{lattimore2020learning,jin2020provably}, we have $( \alpha \| \phi(1,i) \|_{(\Lambda^k)^{-1}} )^2 = O (  K \epsmis^2 / N_k(i)  )$, where $N_k(i)$ is the number of times action $i$ was played in the first $k$ episodes. By \eqref{eqLinUcb}, this means the algorithm needs to explore uniformly for the first $\Omega(K \epsmis^2)$ steps to drive the exploration bonuses down to $O(1)$. In contrast, $( \alpha \| \phi(1,i) \|_{(\Lambda^k)^{-1}} )^2 = O (  K \epsmis^2 / ( K \epsmis^2 + N_k(i) )) = O(1)$ holds right away (i.e., for $k=1$) with our $\lambda$.
\end{rem}

\section{Conclusion}

In this work, we proposed the \texttt{Sup-LSVI-UCB} algorithm and showed it is the first to achieve \ref{itRegret}, \ref{itComp}, and \ref{itPar}. Our algorithm is motivated by a new interpretation of \texttt{Sup-Lin-UCB}, which also helps explain the results of \cite{takemura2021parameter} intuitively. Additionally, we improved existing regret bounds for MLMDPs when only \ref{itPar}, or \ref{itRegret} and \ref{itComp}, are required. We also showed \texttt{Lin-UCB} can be improved in terms of \ref{itRegret} when $\epsmis$ is known, which should extend to \texttt{LSVI-UCB}. 

\textit{Broader societal impact:} While this work is theoretical, it is motivated by the very practical issue of model misspecification. In practice, automated decision making algorithms like ours should be actively monitored to mitigate the risk of biased decisions (which may arise from biased training data, for example).

\section*{Acknowledgements}

This work was partially supported by ONR Grant N00014-19-1-2566, ARO Grant ARO W911NF-19-1-0379, NSF/USDA Grant AG 2018-67007-28379, and NSF Grants 1910112, 2019844, 1704970, and 1934986.

\bibliography{references}

\newpage \appendix \allowdisplaybreaks \onecolumn

\section{\cite{takemura2021parameter}'s algorithm} \label{appCompareSLU}

Algorithm \ref{algTakemura} is the \texttt{Sup-Lin-UCB} variant from \cite{takemura2021parameter} (in our notation). The key differences from Algorithm \ref{algSLU} are twofold. First, at each phase $l$, Algorithm \ref{algTakemura} either chooses an optimistic action (Line \ref{lnTakemuraExploit}), eliminates suboptimal actions (Line \ref{lnTakemuraEliminate}), or chooses an exploratory action (Line \ref{lnTakemuraExplore}). In contrast, Algorithm \ref{algSLU} either explores (Line \ref{lnSLUexplore}) or eliminates (Line \ref{lnSLUeliminate}) for phases $l < L$ and either explores (Line \ref{lnSLUexplore}) or exploits (Line \ref{lnSLUexploit}) in phase $l = L$. Second, Algorithm \ref{algTakemura} uses \texttt{Lin-UCB}-style exploration bonuses in Lines \ref{lnTakemuraExploit} and \ref{lnTakemuraEliminate}, which the corresponding lines of Algorithm \ref{algSLU} do not. In both cases, we made these changes to simplify the algorithm and unify the presentation with \texttt{EXPL3}, and we found this does not worsen regret in an order sense.

\IncMargin{1em}
\begin{algorithm} \label{algTakemura}

\caption{\texttt{Sup-Lin-UCB-Var}}

$\Psi_l^0 = \emptyset\ \forall\ l \in [L]$

\For{episode $k = 1, \ldots , K$}{

Observe $s^k \in \S$, set $l=1$ and $\A_l^k = \A$

\Repeat{$a^k$ is chosen}{

$\Lambda_l^k = I + \sum_{ \tau \in \Psi_l^{k-1} } \phi(s^\tau,a^\tau) \phi(s^\tau,a^\tau)^\trans$, $w_l^k = ( \Lambda_l^k )^{-1} \sum_{ \tau \in \Psi_l^{k-1} } \phi(s^\tau,a^\tau) r^\tau(s^\tau,a^\tau)$

\uIf{$\max_{a \in \A_l^k} \|\phi(s^k,a) \|_{(\Lambda_l^k)^{-1}} \leq \sqrt{d/K}$}{

$a^k = \argmax_{a \in \A_l^k} ( \phi(s^k,a)^\trans w_l^k + \alpha \| \phi(s^k,a) \|_{(\Lambda_l^k)^{-1}} )$, $\Psi_{l'}^k = \Psi_{l'}^{k-1}\ \forall\ l'$ \label{lnTakemuraExploit}

}
\uElseIf{$\max_{a \in \A_l^k} \|\phi(s^k,a) \|_{(\Lambda_l^k)^{-1}} \leq 2^{-l}$}{

$\A_{l+1}^k = \{ a \in \A_l^k : \phi(s^k,a)^\trans w_l^k + \alpha \|\phi(s^k,a) \|_{(\Lambda_l^k)^{-1}} \geq \max_{a' \in \A_l^k} ( \phi(s^k,a')^\trans w_l^k + \alpha \|\phi(s^k,a') \|_{(\Lambda_l^k)^{-1}} ) - 2^{1-l} \alpha  \}$  \label{lnTakemuraEliminate}

}
\uElse{

$a^k \in \{ a \in \A_l^k : \|\phi(s^k,a) \|_{(\Lambda_l^k)^{-1}} > 2^{-l} \}$, $\Psi_l^k = \Psi_l^{k-1} \cup \{k\}$, $\Psi_{l'}^k = \Psi_{l'}^{k-1}\ \forall\ l' \neq l$  \label{lnTakemuraExplore}

}

}

Play $a^k$, observe $r^k(s^k,a^k)$

}

\end{algorithm}
\DecMargin{1em}

\section{Theorem \ref{thmMain} proof} \label{appMainProof}

In this appendix, we prove Theorem \ref{thmMain}. We begin with some basic inequalities in Appendix \ref{appSimple}. We then prove our main concentration result in Appendix \ref{appConcentration}. Next, Appendix \ref{appEstimation} provides a general result for the $Q$-function estimates in Algorithm \ref{algSub}. Using this result, Appendices \ref{appLearner} and \ref{appOptimal} bound the differences $V_h^k(s)  - V_h^{\pi_k}(s)$ and $V_h^\star(s) - V_h^k(s)$, respectively. This yields a bound on the episode $k$ regret $V_1^\star(s_1^k) - V_1^{\pi_k}(s_1^k)$, which we use in Appendix \ref{appRegret} to prove the regret guarantee. Along the way, we defer some proof details to Appendix \ref{appMainProofDetails}, which also contains the complexity analysis.

\subsection{Simple results} \label{appSimple}

We first bound the error that arises from the rounding performed in Algorithm \ref{algMain}.
\begin{clm}[Rounding error] \label{clmTildeEpsrnd}
For any $k \in [K]$, $h \in [H]$, $l \in [L]$, $s \in \S$, and $a \in \A$, we have
\begin{equation}
| \phi(s,a)^\trans ( w_{h,l}^k - \tilde{w}_{h,l}^k ) | \leq \sqrt{d} \epsrnd , \quad \left| \| \phi(s,a) \|_{ (\Lambda_{h,l}^k )^{-1} } - \| \phi(s,a) \|_{ (\tilde{\Lambda}_{h,l}^k )^{-1} } \right| \leq \sqrt{ d \epsrnd }   .
\end{equation}
\end{clm}

Next, we have the following bounds for the bonus terms.
\begin{clm}[Bonuses] \label{clmBonus}
For any $k \in [K]$, $h \in [H]$, $s \in \S$, $l \in [l_h^k(s)-1]$, and $a \in \A_{h,l}^k(s)$, we have $\| \phi(s,a) \|_{ ( \tilde{\Lambda}_{ h , l }^k )^{-1} } \leq 2^{-l}$ and $\| \phi(s,a) \|_{ ( \Lambda_{ h , l }^k )^{-1} } \leq 2^{1-l}$.
\end{clm}
\begin{proof}
The first bound holds by definition in Algorithm \ref{algSub}. For the second bound, we use the first, Claim \ref{clmTildeEpsrnd}, and $\epsrnd \leq 2^{-2l} / d$ in Algorithm \ref{algMain} to obtain $\| \phi(s,a) \|_{ ( \Lambda_{ h , l }^k )^{-1} } \leq \| \phi(s,a) \|_{ ( \tilde{\Lambda}_{ h , l  }^k )^{-1} } + 2^{-l} \leq 2^{-l} + 2^{-l} = 2^{1-l}$.
\end{proof}

Finally, we bound the cardinality of $\Psi_{h,l}^k$. This is an analogue of \cite{takemura2021parameter}'s Lemma 1, which shows $|\Psi_{h,l}^k| = O ( 4^l d \log ( K / d ) )$. With a more careful argument, we obtain a bound that is independent of $K$ (for any fixed $l$), which will be crucial in achieving \ref{itRegret} and \ref{itComp}.
\begin{clm}[Dataset bound] \label{clmPsiBound}
For any $k \in [K]$, $h \in [H]$, and $l \in [L]$, we have $|\Psi_{h,l}^k| \leq 40 \cdot 4^l d l \leq 2^{3l+5} d$.
\end{clm}
\begin{proof}[Proof sketch]
By Lines \ref{lnSubExplore}-\ref{lnSubExploreOther} of Algorithm \ref{algSub} and Line \ref{lnPsiCase1A} of Algorithm \ref{algMain}, for each $\tau \in \Psi_{h,l}^k$, we know that $\| \phi(s_h^\tau,a_h^\tau) \|_{ ( \tilde{\Lambda}_{h,l}^\tau )^{-1} } > 2^{-l}$. This implies $| \Psi_{h,l}^k | \leq 4^l \sum_{\tau \in \Psi_{h,l}^k } \| \phi(s_h^\tau,a_h^\tau) \|_{ ( \tilde{\Lambda}_{h,l}^\tau )^{-1} }^2 = \tilde{O}( 4^l d )$, where the equality follows from Claim \ref{clmTildeEpsrnd} and \cite{abbasi2011improved}. See Appendix \ref{appMainProof} for details.
\end{proof}

\subsection{Concentration} \label{appConcentration}

For any $k \in [K]$, $h \in [H]$, $l \in [L]$, and $V : \S \rightarrow \R$, define the bad event
\begin{equation}
\mathcal{B}(k,h,l,V) = \left\{ \left\| \sum_{ \tau \in \Psi_{h,l}^{k-1} } \phi(s_h^\tau, a_h^\tau ) \left( \eta_h^\tau + V(s_{h+1}^\tau) - \E_{s_{h+1}^\tau} V(s_{h+1}^\tau)\right) \right\|_{ ( \Lambda_{h,l}^k )^{-1} } > \beta \right\} ,
\end{equation} 
where $\E_{s_{h+1}^\tau} V(s_{h+1}^\tau) = \int_{s' \in \S} V(s')  P_h(s'|s_h^\tau,a_h^\tau)$ only averages over $s_{h+1}^\tau$ (even if $V$ is random, in particular, if $V = V_{h+1}^k$) and $\beta = 13 d H L \sqrt{ \log(3dHL/\delta) }$. Also define the good event
\begin{equation}
\mathcal{G} = \cap_{k=1}^K \cap_{h=1}^H \cap_{l=1}^L \mathcal{B} ( k,h,l, V_{h+1}^k )^C .
\end{equation}
As discussed in Remark \ref{remRounding}, a similar event is analyzed in \cite{jin2020provably} using covering arguments. In contrast, here $V_{h+1}^k$ belongs to a finite function class, which allows us to show that $\mathcal{G}$ occurs with high probability via a direct union bound over the function class.
\begin{lem}[Concentration] \label{lemConcentration}
The good event $\mathcal{G}$ occurs with probability at least $1-\delta/2$.
\end{lem}
\begin{proof}
We fix $h \in [H]$ and $l \in [L]$ and show $\P ( \cup_{k=1}^K \mathcal{B}(k,h,l,V_{h+1}^k) ) \leq \frac{\delta}{2HL}$, which (by the union bound) completes the proof. Toward this end, we introduce some notation. For any (ordered) sets $\mathscr{X} = \{ x_{l'} \}_{l'=1}^L \subset \R^d$ and $\mathscr{Y} = \{ Y_{l'} \}_{l'=1}^L \subset \R^{d \times d}$, let $V_{ \mathscr{X}, \mathscr{Y} } : \S \rightarrow [0,H]$ be the function that results from running Algorithm \ref{algSub} with $\tilde{w}_{h+1,l'}^k$ and $(\tilde{\Lambda}_{h+1,l'}^k )^{-1}$ replaced by $x_{l'}$ and $Y_{l'}$, respectively. Hence, if $\mathscr{X} = \{  \tilde{w}_{h+1,l'}^k \}_{l'=1}^L$ and $\mathscr{Y} = \{ ( \tilde{\Lambda}_{h+1,l'}^k )^{-1} \}_{l'=1}^L$, then $V_{\mathscr{X}, \mathscr{Y}} = V_{h+1}^k$. Next, define 
\begin{gather}
\mathcal{X} = \{ [ \epsrnd i_j ]_{j=1}^d : i_j \in \{ - \ceil{ ( 2^L d H )^4 / \epsrnd } , \ldots , \ceil{ ( 2^L d H )^4 / \epsrnd } \}\ \forall\ j \in [d]  \} \subset \R^d , \\
\mathcal{Y}  = \{ [ \epsrnd i_{j_1,j_2} ]_{j_1,j_2=1}^d : i_{j_1,j_2} \in \{ - \ceil{ 1 / (16 \epsrnd) } , \ldots , \ceil{ 1 / (16 \epsrnd) } \}\ \forall\ (j_1,j_2) \in [d]^2 \} \subset \R^{d \times d} .
\end{gather}
By Claim \ref{clmWl2Bound} in Appendix \ref{appMainProof} (which shows $\|w_{h,l}^k\|_\infty \leq (2^L dH)^4$), we have $\{ \tilde{w}_{h+1,l'}^k \}_{l'=1}^L \in \mathcal{X}^L$. Further, by a standard matrix norm inequality and the fact that the eigenvalues of $\Lambda_{h+1,l'}^k$ are at least $16$, we have $\max_{j_1,j_2} | ( \Lambda_{h+1,l'}^k )^{-1}_{j_1,j_2} | \leq \| ( \Lambda_{h+1,l'}^k )^{-1} \|_2 \leq 1/16$, so $\{ ( \tilde{\Lambda}_{h+1,l'}^k )^{-1} \}_{l'=1}^L \in \mathcal{Y}^L$. Finally, we know $V_{h+1}^k : \S \rightarrow [0,H]$ by Lines \ref{lnSubExploreOther} and \ref{lnSubExploitOther} of Algorithm \ref{algSub}. Thus, if we define $\mathcal{V} = \{ V_{\mathscr{X},\mathscr{Y}} : \S \rightarrow [0,H] |  \mathscr{X} \in \mathcal{X}^L , \mathscr{Y} \in \mathcal{Y}^L\}$, then $V_{h+1}^k \in \mathcal{V}\ \forall\ k \in [K]$, which implies 
$\cup_{k=1}^K \mathcal{B}(k,h,l,V_{h+1}^k) \subset  \cup_{ V \in \mathcal{V} } \cup_{k=1}^K \mathcal{B}(k,h,l,V)$. Hence, taking another union bound, it suffices to show that for any $V \in \mathcal{V}$, $\P ( \cup_{k=1}^K \mathcal{B}(k,h,l,V) ) \leq \frac{\delta}{ 2 H L |\mathcal{V}|}$. Let $\ind(\cdot)$ denote the indicator function, and for each $\tau \in [K]$, define the folllowing:
\begin{gather}
\underline{s}_{\tau+1} = s_{h+1}^\tau , \quad \underline{\phi}_\tau =  \phi(s_h^\tau,a_h^\tau) \ind (  \tau \in \Psi_{h,l}^\tau ) , \quad \underline{\Lambda}_k = 16 I  + \sum_{\tau=1}^{k-1} \underline{\phi}_\tau \underline{\phi}_\tau^\trans , \quad \upsilon_\tau = \eta_h^\tau  + V(\underline{s}_{\tau+1}) - \E_{s_{h+1}^\tau} V(s_{h+1}^\tau) .
\end{gather}
Then by definition, for any $k \in [K]$, we have 
\begin{equation} \label{eqConcRewrite}
\Lambda_{h,l}^k = \underline{\Lambda}_k  , \quad \left\| \sum_{ \tau \in \Psi_{h,l}^{k-1} } \phi(s_h^\tau, a_h^\tau ) \left( \eta_h^\tau + V(s_{h+1}^\tau) - \E_{s_{h+1}^\tau} V(s_{h+1}^\tau)  \right) \right\|_{ ( \Lambda_{h,l}^k )^{-1} } = \left\| \sum_{ \tau =1 }^{k-1} \underline{\phi}_\tau \upsilon_\tau \right\|_{ \underline{\Lambda}_k^{-1} }  .
\end{equation}
Also let $\mathcal{F}_0 = \emptyset$ and $\mathcal{F}_\tau = \sigma ( \mathcal{F}_{\tau-1} \cup \sigma ( s_1^\tau , a_1^\tau , \eta_1^\tau , \ldots , s_{h-1}^\tau , a_{h-1}^\tau , \eta_{h-1}^\tau , s_h^\tau , a_h^\tau  ) )$ for each $\tau \in \N$, where $\sigma(\cdot)$ is the generated $\sigma$-algebra. Hence, in words, $\mathcal{F}_\tau$ contains all randomness until the random reward and next state are realized at step $h$ of episode $\tau$. Note $\underline{\phi}_\tau$ is $\mathcal{F}_\tau$-measurable and $\upsilon_\tau$ is $\mathcal{F}_{\tau+1}$-measurable with $\E [ \upsilon_\tau | \mathcal{F}_\tau ] = 0$. Furthermore, since $\eta_h^\tau \in [-1,1]$ by assumption (see Section \ref{secPrelim}) and $V : \S \rightarrow [0,H]$ by definition, we have $\upsilon_\tau \in [-2H,2H]$, so $\upsilon_\tau$ is $(2H)$-subgaussian. Therefore,
\begin{align}
\P ( \cup_{k=1}^K \mathcal{B}(k,h,l,V)) & = \P \left( \cup_{k=1}^K \left\{ \left\| \sum_{ \tau =1 }^{k-1} \underline{\phi}_\tau \upsilon_\tau \right\|_{ \underline{\Lambda}_k^{-1} } > \beta \right\} \right) \\
& \leq \P \left( \cup_{k=1}^K \left\{ \left\| \sum_{ \tau =1 }^{k-1} \underline{\phi}_\tau \upsilon_\tau \right\|_{ \underline{\Lambda}_k^{-1} } > \sqrt{ 8 H^2 \log \left(  \frac{ \det ( \Lambda_{h,l}^k ) }{ \det ( 16 I )} \frac{2 H L |\mathcal{V}|}{\delta} \right) } \right\} \right)  \leq \frac{\delta}{ 2 H L |\mathcal{V}| } ,
\end{align}
where the equality uses \eqref{eqConcRewrite}, the first inequality is a simple calculation (see Claim \ref{clmBetaBound} in Appendix \ref{appMainProof} for details), and the second inequality is Theorem 1 from \cite{abbasi2011improved}. 
\end{proof}

\begin{rem}[MLCB] \label{remMLCB}
When $H=1$, we simply have $\upsilon_\tau = \eta_h^\tau$ in the proof of Lemma \ref{lemConcentration}, so we do not require a union bound over $\mathcal{V}$. This union bound makes $\beta$, and subsequently $\alpha$, have linear (instead of square root) dependence on $d$, which in turn gives the $\sqrt{K}$ term in our regret bound $d^{3/2}$ (instead of $d$) dependence.
\end{rem}

\subsection{Estimation error} \label{appEstimation}

For the remainder of the proof, we bound regret on the good event $\mathcal{G}$. We first show that on $\mathcal{G}$, the least-squares estimate $w_{h,l}^k$ is close to $\bar{w}_h^k \triangleq \theta_h + \int_{s' \in \S} V_{h+1}^k(s') \mu_h(s')$ in a certain sense. For this, it will be convenient to introduce the following notation:
\begin{equation}
\Delta_h^r(s,a) = r_h(s,a) - \phi(s,a)^\trans \theta_h , \quad \Delta_h^P( s' | s,a ) = P_h(s'|s,a) - \phi(s,a)^\trans \mu_h .
\end{equation}
(Note $| \Delta_h^r(s,a) | , \| \Delta_h^P(\cdot|s,a) \|_1 \leq \epsmis$ by Assumption \ref{assLin}.)  We can now prove a generalization of \cite{takemura2021parameter}'s Lemma 2 using an approach somewhat similar to \cite{jin2020provably}'s Lemma B.4.
\begin{clm}[Least-squares error] \label{clmEstError}
On the event $\mathcal{G}$, for any $k \in [K]$, $h \in [H]$, $s \in \S$, $l \in [l_h^k(s)-1]$, and $a \in \A_{h,l}^k(s)$, we have $| \phi(s,a)^\trans ( w_{h,l}^k - \bar{w}_h^k ) | \leq 2^{-l} \alpha + 26 H \sqrt{ d l } \epsmis$.
\end{clm}
\begin{proof}
By definition of $\Lambda_{h,l}^k$, we have
\begin{equation} \label{eqBarWeightsA}
\bar{w}_h^k = ( \Lambda_{h,l}^k )^{-1}  \Lambda_{h,l}^k \bar{w}_h^k = 16 ( \Lambda_{h,l}^k )^{-1} \bar{w}_h^k +  ( \Lambda_{h,l}^k )^{-1} \sum_{ \tau \in \Psi_{h,l}^{k-1} } \phi(s_h^\tau, a_h^\tau)  \phi(s_h^\tau, a_h^\tau)^\trans \bar{w}_h^k .
\end{equation}
By definition of $\bar{w}_h^k$ and Assumption \ref{assLin}, for any $\tau \in \Psi_{h,l}^{k-1}$, we know
\begin{equation} \label{eqBarWeightsB}
\phi(s_h^\tau, a_h^\tau)^\trans \bar{w}_h^k = r_h (s_h^\tau, a_h^\tau) + \E_{s_{h+1}^\tau} V_{h+1}^k(s_{h+1}^\tau)  - \Delta_h^r (s_h^\tau, a_h^\tau) - \int_{s' \in \S} V_{h+1}^k(s') \Delta_h^P(s'|s_h^\tau, a_h^\tau)  .
\end{equation}
Additionally, recall that in Algorithm \ref{algMain}, we have
\begin{equation}\label{eqBarWeightsC}
w_{h,l}^k= ( \Lambda_{h,l}^k )^{-1} \sum_{ \tau \in \Psi_{h,l}^{k-1} } \phi(s_h^\tau,a_h^\tau) ( r_h(s_h^\tau,a_h^\tau) + \eta_h^\tau + V_{h+1}^k ( s_{h+1}^\tau )  ) .
\end{equation}
It follows that $\phi(s,a)^\trans ( w_{h,l}^k - \bar{w}_h^k ) = \sum_{i=1}^3 \phi(s,a)^\trans (\Lambda_{h,l}^k )^{-1} z_i$, where we define
\begin{gather}
z_1 =\sum_{ \tau \in \Psi_{h,l}^{k-1} } \phi(s_h^\tau, a_h^\tau) \left( \eta_h^\tau + V_{h+1}^k(s_{h+1}^\tau) - \E_{s_{h+1}^\tau} V_{h+1}^k ( s_{h+1}^\tau ) \right)  , \quad z_2 = - 16  \bar{w}_h^k , \\
z_3 =  \sum_{ \tau \in \Psi_{h,l}^{k-1} } \phi(s_h^\tau, a_h^\tau) \left( \Delta_h^r(s_h^\tau,a_h^\tau) + \int_{s' \in \S} V_{h+1}^k(s')  \Delta_{h,s'}^P ( s_h^\tau, a_h^\tau) \right) .
\end{gather}
Hence, we aim to bound $|\phi(s,a)^\trans (\Lambda_{h,l}^k )^{-1} z_i|$ for each $i \in [3]$. By Cauchy-Schwarz, on the event $\mathcal{G}$,
\begin{equation} \label{eqZ1final}
| \phi(s,a)^\trans (\Lambda_{h,l}^k )^{-1} z_1 | \leq \| z_1 \|_{ (\Lambda_{h,l}^k )^{-1} } \| \phi(s,a) \|_{ (\Lambda_{h,l}^k )^{-1} } \leq \beta \| \phi(s,a) \|_{ (\Lambda_{h,l}^k )^{-1} } .
\end{equation}
Again using Cauchy-Schwarz, we have
\begin{equation} \label{eqZ2final}
| \phi(s,a)^\trans (\Lambda_{h,l}^k )^{-1} z_2 | \leq  16 \|  \bar{w}_h^k \|_{ (\Lambda_{h,l}^k )^{-1} } \| \phi(s,a) \|_{ (\Lambda_{h,l}^k )^{-1} } \leq 8 \sqrt{d} H \| \phi(s,a) \|_{ (\Lambda_{h,l}^k )^{-1} } ,
\end{equation}
where the second inequality holds because, by Claim \ref{clmNormEq} in Appendix \ref{appMainProof} (a simple norm equivalence), Assumption \ref{assLin}, and the fact that $V_{h+1}^k : \S \rightarrow [0,H]$ in Algorithm \ref{algSub},
\begin{equation}
\| \bar{w}_h^k \|_{( \Lambda_{h,l}^k )^{-1} } \leq \| \bar{w}_h^k \|_2 / 4 \leq \sqrt{d} ( 1 +  H ) / 4 \leq \sqrt{d} H / 2 .
\end{equation}
For $z_3$, first note that by Assumption \ref{assLin} and Algorithm \ref{algSub},
\begin{equation} 
\left| \Delta_h^r(s_h^\tau,a_h^\tau) + \int_{s' \in \S} V_{h+1}^k(s') \Delta_h^P (s' |  s_h^\tau, a_h^\tau)  \right| \leq  \left( 1 + \max_{s' \in \S} V_{h+1}^k(s') \right) \epsmis \leq (1+H) \epsmis \leq 2 H \epsmis .
\end{equation}
Furthermore, since $l \leq l_h^k(s) - 1$ and $a \in \A_{h,l}^k(s)$, we can use Claims \ref{clmBonus} and \ref{clmPsiBound} to obtain
\begin{equation} 
| \Psi_{h,l}^{k-1}  | \times \phi(s,a)^\trans ( \Lambda_{h,l}^k )^{-1} \phi(s,a) \leq  40 \cdot 4^l d l \times 4^{1-l} = 160 d l < 13^2 d l .
\end{equation}
By the previous two bounds, Cauchy-Schwarz, and positive-semidefiniteness, we obtain
\begin{align}  \label{eqZ3final}
| \phi(s,a)^\trans ( \Lambda_{h,l}^k )^{-1}  z_3 | & \leq 2 H \epsmis \sum_{ \tau \in \Psi_{h,l}^{k-1} } | \phi(s,a)^\trans ( \Lambda_{h,l}^k )^{-1} \phi(s_h^\tau, a_h^\tau) |  \\
& \leq 2 H \epsmis \sqrt{ | \Psi_{h,l}^{k-1}  | \phi(s,a)^\trans ( \Lambda_{h,l}^k )^{-1} \sum_{ \tau \in \Psi_{h,l}^{k-1} } \phi(s_h^\tau, a_h^\tau) \phi(s_h^\tau, a_h^\tau) ( \Lambda_{h,l}^k )^{-1}  \phi(s,a)  } \\
& = 2 H \epsmis \sqrt{ | \Psi_{h,l}^{k-1}  | \phi(s,a)^\trans ( \Lambda_{h,l}^k )^{-1} ( \Lambda_{h,l}^k - 16 I ) ( \Lambda_{h,l}^k )^{-1}  \phi(s,a)  }  \\
& \leq  2 H \epsmis \sqrt{ | \Psi_{h,l}^{k-1}  | \phi(s,a)^\trans ( \Lambda_{h,l}^k )^{-1} \phi(s,a) } \leq 26 H \epsmis \sqrt{  d l  } .
\end{align}
Hence, combining \eqref{eqZ1final}, \eqref{eqZ2final}, and \eqref{eqZ3final}, we obtain
\begin{equation}
| \phi(s,a)^\trans ( w_{h,l}^k - \bar{w}_h^k ) | \leq \sum_{i=1}^3 |\phi(s,a)^\trans (\Lambda_{h,l}^k )^{-1} z_i| \leq ( \beta + 8 \sqrt{d} H ) \| \phi(s,a) \|_{ ( \Lambda_{h,l}^k )^{-1} } + 26 H \epsmis \sqrt{ d l } .
\end{equation}
This completes the proof, because $\| \phi(s,a) \|_{ ( \Lambda_{h,l}^k )^{-1} }  \leq 2^{1-l}$ by Claim \ref{clmBonus}, and by definition,
\begin{align}
& \beta + 8 \sqrt{d} H =  13 d H L \sqrt{ \log(3dHL/\delta) } + 8 \sqrt{d} H \leq  21 d H L \sqrt{ \log(3dHL/\delta) } = \alpha / 2 . \qedhere
\end{align}
\end{proof}

We conclude this subsection by using Claim \ref{clmEstError} to show the $Q$-function estimates in Algorithm \ref{algSub} are close to the function $\bar{Q}_h^k : \S \times \A \rightarrow [0, 2H]$ defined by 
\begin{equation}
\bar{Q}_h^k(s,a) = r_h(s,a) + \int_{s' \in \S} V_{h+1}^k(s') P_h(s'|s,a) \ \forall\ (s,a) \in \S \times \A .
\end{equation}
It will also be convenient to define $\epseff = 2 \alpha \sqrt{d \epsrnd } + 28 H \sqrt{ d L } \epsmis$, which is the effective misspecification (true misspecification and rounding error) that we carry through the next portion of the proof.
\begin{lem}[$Q$-function error] \label{lemEstError}
On the event $\mathcal{G}$, for any $k \in [K]$, $h \in [H]$, $s \in \S$, $l \in [l_h^k(s) - 1]$, and $a \in \A_{ h , l }^k(s)$, we have $| \phi(s,a)^\trans \tilde{w}_{h,l}^k - \bar{Q}_h^k(s,a) | \leq 2^{-l} \alpha + \epseff$.
\end{lem}
\begin{proof}
By the triangle inequality, we have
\begin{align}
 | \phi(s,a)^\trans \tilde{w}_{h,l}^k - \bar{Q}_h^k(s,a) | |\phi(s,a)^\trans ( \tilde{w}_{h,l}^k - w_{h,l}^k ) |  +  | \phi(s,a)^\trans ( w_{h,l}^k - \bar{w}_h^k )  | + | \phi(s,a)^\trans \bar{w}_h^k - \bar{Q}_h^k(s,a) | .
\end{align}
For the first term, by Claim \ref{clmTildeEpsrnd}, and since $\alpha \geq 1$ and $\epsrnd \in (0,1)$,
\begin{equation}
| \phi(s,a)^\trans ( \tilde{w}_{h,l}^k - w_{h,l}^k ) | \leq \sqrt{d} \epsrnd \leq \alpha \sqrt{ d \epsrnd } .
\end{equation}
For the second term, by Claim \ref{clmEstError} and since $l \leq l_h^k(s)-1 \leq L$,
\begin{align}
| \phi(s,a)^\trans ( w_{h,l}^k - \bar{w}_h^k ) | \leq 2^{-l} \alpha + 26 H \sqrt{ d l } \epsmis \leq 2^{-l} \alpha + 26 H \sqrt{ d L } \epsmis
\end{align}
For the third term, by Assumption \ref{assLin} and Algorithm \ref{algSub},
\begin{align}
| \bar{Q}_h^k(s,a) - \phi(s,a)^\trans \bar{w}_h^k |  = \left| \Delta_h^r(s,a) + \int_{s' \in \S} \Delta_h^P(s'|s,a) V_{h+1}^k(s') \right| \leq 2 H \epsmis \leq 2 H \sqrt{ d L } \epsmis .
\end{align}
Hence, combining all of the above, we obtain
\begin{align}
& | \phi(s,a)^\trans \tilde{w}_{h,l}^k - \bar{Q}_h^k(s,a) | \leq 2^{-l} \alpha + 2 \alpha \sqrt{ d \epsrnd } + 28 H \sqrt{ dL } \epsmis = 2^{-l} \alpha + \epseff . \qedhere
\end{align}
\end{proof}

\subsection{Algorithm policy error} \label{appLearner}

Our next goal is to bound the difference between the value function estimate $V_h^k$ and the true value function $V_h^{\pi_k}$ of the algorithm's policy. We begin with an intermediate result. This is roughly an analogue of \cite{takemura2021parameter}'s Lemma 5 and 8, though our streamlined approach yields a single result. Additionally, we have to deal with the projection in Algorithm \ref{algSub}, which complicates the proof.
\begin{clm}[Algorithm error, one-step] \label{clmEstErrorLearner}
On the event $\mathcal{G}$, for any $k \in [K]$, $h \in [H]$, and $s \in \S$, we have $V_h^k(s) - \bar{Q}_h^k(s, \pi_h^k(s) ) \leq 8 \alpha \cdot 2^{-l_h^k(s)} + \epseff$.
\end{clm}
\begin{proof}
Let $l = l_h^k(s) - 1$. By Corollary \ref{corExploratory} from Appendix \ref{appMainProof}, we know that $l \in [L]$. We first assume $l \in [L-1]$, which implies $V_h^k(s) = \mathcal{P}_{[0,H]} ( V_{h,l}^k(s) + 2^{-l} \alpha )$ in Algorithm \ref{algSub}. Hence, if $V_{h,l}^k(s) < -2^{-l} \alpha$, then $V_h^k(s) = 0$, which immediately yields the desired bound (since $\bar{Q}_h^k(s, \pi_h^k(s) ) \geq 0$). If instead $V_{h,l}^k(s) \geq -2^{-l} \alpha$, then $V_h^k(s) \leq V_{h,l}^k(s) + 2^{-l} \alpha$, so it suffices to prove the bound with $V_h^k(s)$ replaced by $V_{h,l}^k(s) + 2^{-l} \alpha$. Toward this end, first observe that by Lemma \ref{lemEstError} and since $\pi_h^k(s) \in \A_{h,l+1}^k(s) \subset \A_{h,l}^k(s)$,
\begin{equation}
\phi(s,\pi_h^k(s))^\trans \tilde{w}_{h,l}^k \leq \bar{Q}_h^k(s,\pi_h^k(s)) + 2^{-l} \alpha + \epseff .
\end{equation}
On the other hand, again using $\pi_h^k(s) \in \A_{h,l+1}^k(s)$, we know
\begin{equation}
V_{h,l}^k(s) + 2^{-l} \alpha \leq  \phi(s,\pi_h^k(s))^\trans \tilde{w}_{h,l}^k + 2^{1-l} \alpha  + 2^{-l} \alpha .
\end{equation}
Hence, combining the inequalities, we obtain
\begin{equation}
V_{h,l}^k(s) + 2^{-l} \alpha - \bar{Q}_h^k(s,\pi_h^k(s)) \leq 2^{1-l} \alpha  + 2^{-l} \alpha + 2^{-l} \alpha + \epseff  = 8 \alpha \cdot 2^{-l_h^k(s)} + \epseff .
\end{equation}

For $l = L$, we have $V_h^k(s) = \mathcal{P}_{[0,H]} ( V_{h,L}^k(s) )$ in Algorithm \ref{algSub}. If $V_{h,L}^k(s) < 0$, the bound is again immediate. If instead $V_{h,L}^k(s) \geq 0$, we know $V_h^k(s) \leq V_{h,L}^k(s)$, so we can prove the bound with $V_h^k(s)$ replaced by $V_{h,L}^k(s)$. By Algorithm \ref{algSub} and Lemma \ref{lemEstError}, we have
\begin{align}
& V_{h,L}^k(s) - \bar{Q}_h^k(s,\pi_h^k(s)) = \phi( s , \pi_h^k(s) )^\trans \tilde{w}_{h,L}^k - \bar{Q}_h^k(s,\pi_h^k(s)) \leq 2^{-L} \alpha + \epseff < 8 \alpha \cdot 2^{-l_h^k(s)} + \epseff  , 
\end{align}
where the final inequality is $2^{-L} = 2^{1-l_h^k(s)} < 8 \cdot 2^{-l_h^k(s)}$.
\end{proof}

Next, for any $k \in [K]$ and $h \in [H]$, define the martingale noise term
\begin{equation}
\gamma_h^k = \E_{ s_{h+1}^k } ( V_{h+1}^k(s_{h+1}^k) - V_{h+1}^{\pi_k} ( s_{h+1}^k ) ) - ( V_{h+1}^k(s_{h+1}^k) - V_{h+1}^{\pi_k} (s_{h+1}^k ) ) .
\end{equation}
Using the previous claim and a simple inductive argument, we can prove the following lemma (see Appendix \ref{appMainProof} for details). In essence, similar to \cite{jin2020provably}'s Lemma B.6, this lemma shows that the noise $\gamma_h^k$ in the backward induction yields a martingale difference sequence.
\begin{lem}[Algorithm error, multi-step] \label{lemEstErrorLearner}
On the event $\mathcal{G}$, for any $k \in [K]$ and any $h \in [H]$, we have
\begin{equation} \label{eqRecursion}
V_h^k (s_h^k) - V_h^{\pi_k}(s_h^k) \leq 8 \alpha \sum_{h'=h}^H 2^{ - l_{h'}^k(s_{h'}^k) }  + \sum_{h'=h}^H \gamma_{h'}^k + (H-h+1) \epseff  .
\end{equation}
\end{lem}

\subsection{Optimal policy error} \label{appOptimal}

Next, we bound the difference between the optimal value function $V_h^\star$ and the value function estimate $V_h^k$. We start with two intermediate results. First, for each $k \in [K]$, $h \in [H]$, $s \in \S$, and $l \in [l_h^k(s) \wedge L]$, let $\bar{\pi}_{h,l}^k(s) = \argmax_{a \in \A_{h,l}^k(s) } \bar{Q}_h^k(s,a)$. We can then generalize \cite{takemura2021parameter}'s Lemma 4.
\begin{clm}[Error across phases] \label{clmPhaseDecay}
On the event $\mathcal{G}$, for any $k \in [K]$, $h \in [H]$, $s \in \S$, and $l \in [l_h^k(s) \wedge L]$, we have $\bar{Q}_h^k ( s , \bar{\pi}_{h,1}^k(s) ) - \bar{Q}_h^k ( s , \bar{\pi}_{h,l}^k(s) ) \leq 2 (l-1)  \epseff$.
\end{clm}
\begin{proof} 
We use induction on $l$. For $l=1$, the bound holds with equality. Assuming it holds for $l \in [(l_h^k(s)\wedge L)-1]$, we prove it for $l+1$. By the inductive hypothesis, it suffices to show
\begin{equation} \label{eqLinDecayInLKey}
\bar{Q}_h^k ( s , \bar{\pi}_{h,l}^k(s))  - \bar{Q}_h^k ( s , \bar{\pi}_{h,l+1}^k(s) ) \leq 2 \epseff .
\end{equation}
If $\bar{\pi}_{h,l}^k(s) \in \A_{h,l+1}^k(s)$, then since $\A_{h,l+1}^k(s) \subset \A_{h,l}^k(s)$, we have  $\bar{\pi}_{h,l+1}^k(s) = \bar{\pi}_{h,l}^k(s)$ by definition, so \eqref{eqLinDecayInLKey} is immediate. Hence, it only remains to prove \eqref{eqLinDecayInLKey} when $\bar{\pi}_{h,l}^k(s) \notin \A_{h,l+1}^k(s)$. Since $\pi_{h,l}^k(s) \in \A_{h,l+1}^k(s)$ in Algorithm \ref{algSub}, the definition of $\bar{\pi}_{h,l+1}^k(s)$ implies
\begin{equation}
\bar{Q}_h^k ( s , \bar{\pi}_{h,l}^k(s) ) - \bar{Q}_h^k ( s , \bar{\pi}_{h,l+1}^k(s) ) \leq \bar{Q}_h^k ( s , \bar{\pi}_{h,l}^k(s) ) - \bar{Q}_h^k ( s , \pi_{h,l}^k(s) ) .
\end{equation} 
By $l \leq l_h^k(s)-1$, Lemma \ref{lemEstError}, the assumption that $\bar{\pi}_{h,l}^k(s) \notin \A_{h,l+1}^k(s)$, and Algorithm \ref{algSub}, we have
\begin{align}
\bar{Q}_h^k ( s , \bar{\pi}_{h,l}^k(s) ) - \bar{Q}_h^k ( s , \pi_{h,l}^k(s) ) & \leq \phi ( s , \bar{\pi}_{h,l}^k(s) )^\trans \tilde{w}_{h,l}^k - \phi ( s , \pi_{h,l}^k(s) )^\trans \tilde{w}_{h,l}^k + 2^{1-l} \alpha + 2 \epseff < 2 \epseff .
\end{align}
Combining the previous two inequalities, we obtain the desired bound \eqref{eqLinDecayInLKey}.
\end{proof}

As an immediate corollary, we have the following.
\begin{cor}[Error across phases] \label{corPhaseDecay}
On the event $\mathcal{G}$, for any $k \in [K]$, $h \in [H]$, $s \in \S$, and $l \in [l_h^k(s) \wedge L]$, we have $\bar{Q}_h^k ( s , \bar{\pi}_{h,1}^k(s) ) - \bar{Q}_h^k ( s , \bar{\pi}_{h,l}^k(s) ) \leq 2 (L-1) \epseff$.
\end{cor}

We can now bound the difference between $V_h^\star$ and $V_h^k$ in terms of the difference at the next step, i.e., between $V_{h+1}^\star$ and $V_{h+1}^k$. This is similar in spirit to \cite{jin2020provably}'s Lemma B.5.
\begin{clm}[Optimal error, one-step] \label{clmOptimismH}
On the event $\mathcal{G}$, for any $k \in [K]$, $h \in [H]$, and $s \in \S$,
\begin{equation}
V_h^\star(s) - V_h^k(s) \leq \max \left\{ \max_{ a \in \A } \int_{s' \in \S} ( V_{h+1}^\star(s') - V_{h+1}^k(s') ) P_h(s'|s,a) + 2L\epseff + 2^{-L} \alpha , 0 \right\} .
\end{equation}
\end{clm}
\begin{proof}
Let $l = l_h^k(s)-1 \in [L]$. Note that if (1) $l \in [L-1]$ and $V_{h,l}^k(s) + 2^{-l} \alpha > H$ or (2) $l = L$ and $V_{h,l}^k(s) > H$, then $V_h^k(s) = H$, so the bound $V_h^\star(s) - V_h^k(s) \leq 0$ holds. Hence, we assume for the remainder of the proof that either (3) $l \in [L-1]$ and $V_{h,l}^k(s) + 2^{-l} \alpha \leq H$ or (4) $l = L$ and $V_{h,l}^k(s) \leq H$. By definition of $V_h^\star$ and $\bar{Q}_h^k$, we have
\begin{align} 
V_h^\star(s) & = \bar{Q}_h^k ( s , \pi_h^\star(s) )  + \int_{s' \in \S} ( V_{h+1}^\star(s') - V_{h+1}^k(s') ) P_h(s'|s,\pi_h^\star(s)) \\
& \leq \bar{Q}_h^k ( s , \pi_h^\star(s) ) + \max_{a \in \A} \int_{s' \in \S} ( V_{h+1}^\star(s') - V_{h+1}^k(s') ) P_h(s'|s,a)
\end{align}
By definition of $\bar{\pi}_{h,1}^k(s)$, since $\A_{h,1}^k(s) = \A$, and by Corollary \ref{corPhaseDecay},
\begin{equation} 
\bar{Q}_h^k ( s , \pi_h^\star(s) ) \leq \bar{Q}_h^k ( s , \bar{\pi}_{h,1}^k(s) ) \leq \bar{Q}_h^k ( s , \bar{\pi}_{h,l}^k(s) ) + 2 (L-1)  \epseff .
\end{equation}
Again using the definition of $\bar{\pi}_{h,l}^k(s)$, along with Lemma \ref{lemEstError}, we know that
\begin{align} 
\bar{Q}_h^k ( s , \bar{\pi}_{h,l}^k(s) )  = \max_{ a \in \A_{h,l}^k(s) } \bar{Q}_h^k ( s , a ) \leq \max_{a \in \A_{h,l}^k(s) } \phi(s,a)^\trans \tilde{w}_{h,l}^k + 2^{-l} \alpha  +  \epseff = V_{h,l}^k(s) + 2^{-l} \alpha  +  \epseff  .
\end{align}
Hence, stringing together the inequalities, we obtain
\begin{equation} \label{eqOptimismHString}
V_h^\star(s) \leq \max_{a \in \A} \int_{s' \in \S} ( V_{h+1}^\star(s') - V_{h+1}^k(s') ) P_h(s'|s,a) + V_{h,l}^k(s) + 2^{-l} \alpha + 2 L \epseff .
\end{equation}
Now in case (3), we have $V_h^k(s) = \mathcal{P}_{[0,H]}( V_{h,l}^k(s) + 2^{-l} \alpha )$ and $V_{h,l}^k(s) + 2^{-l} \alpha \leq H$, which together imply
\begin{equation} \label{eqOptimismHTwoCase}
V_{h,l}^k(s) + 2^{-l} \alpha \leq V_h^k(s) < V_h^k(s) + 2^{-L} \alpha .
\end{equation}
In case (4), we have $V_h^k(s) = \mathcal{P}_{[0,H]}( V_{h,l}^k(s) )$ and $V_{h,l}^k(s) \leq H$, which implies $V_{h,l}^k(s) \leq V_h^k(s)$. Hence, because $L = l$ in case (4), we again have \eqref{eqOptimismHTwoCase}. Therefore, combining \eqref{eqOptimismHString} and \eqref{eqOptimismHTwoCase}, we obtain
\begin{align}
V_h^\star(s) - V_h^k(s) & \leq \max_{a \in \A} \int_{s' \in \S} ( V_{h+1}^\star(s') - V_{h+1}^k(s') ) P_h(s'|s,a) +  2 L \epseff + 2^{-L} \alpha .  \qedhere 
\end{align}
\end{proof}

Finally, a simple inductive argument yields the following. See Appendix \ref{appMainProof} for details.
\begin{lem}[Optimal error, multi-step] \label{lemOptimism}
On the event $\mathcal{G}$, for any $k \in [K]$, $h \in [H]$, and $s \in \S$, we have $V_h^\star(s) - V_h^k(s) \leq ( 2 L \epseff + 2^{-L} \alpha ) ( H-h+1)$.
\end{lem}

\subsection{Regret bound} \label{appRegret}

First observe that by Algorithms \ref{algMain}-\ref{algSub} and Corollary \ref{corExploratory} from Appendix \ref{appMainProofDetails}, for any $h \in [H]$, we have
\begin{align}
\sum_{k=1}^K 2^{ - l_h^k(s_h^k) }  & = \sum_{k=1}^K \sum_{l=2}^{L+1} 2^{-l} \ind (  l_h^k(s_h^k) = l ) = \sum_{l=2}^L 2^{-l} \sum_{k=1}^K \ind ( l_h^k(s_h^k) = l ) + 2^{-(L+1)} \sum_{k=1}^K \ind ( l_h^k(s_h^k) =L+1) \\
& = \sum_{l=2}^L 2^{-l}  | \Psi_{h,l}^K | + 2^{-(L+1)} \sum_{k=1}^K \ind ( l_h^k(s_h^k) =L+1)  \leq \sum_{l=2}^L 2^{-l}  | \Psi_{h,l}^K | + 2^{-(L+1)} K .
\end{align}
Combined with Lemmas \ref{lemEstErrorLearner} and \ref{lemOptimism}, on the event $\mathcal{G}$, we obtain
\begin{align} \label{eqRegDecomp}
R(K) & \leq 8 \alpha \sum_{h=1}^H \sum_{k=1}^K 2^{- l_h^k(s_h^k) }  + \sum_{k=1}^K \sum_{h=1}^H \gamma_h^k + ( 2^{-L} \alpha  + 3 L \epseff ) HK \\
& \leq 8 \alpha \sum_{h=1}^H \sum_{l=2}^L 2^{-l}  | \Psi_{h,l}^K | + \sum_{k=1}^K \sum_{h=1}^H \gamma_h^k + ( 2^{2-L} \alpha  +2^{-L} \alpha  + 3 L \epseff ) HK  .
\end{align}
For the first summation, by Claim \ref{clmPsiBound} and a simple geometric series computation, for any $h \in [H]$,
\begin{align} 
\sum_{l=2}^L 2^{ - l } | \Psi_{h,l}^K |  & \leq 40 d L \sum_{l=2}^L 2^{-l} \cdot 4^l  = 40 d L \sum_{l=2}^L 2^l < 80 \cdot 2^L d L .
\end{align}
Alternatively, we can use Cauchy-Schwarz and Claim \ref{clmPsiBound} to obtain
\begin{align}
\sum_{l=2}^L 2^{ - l } | \Psi_{h,l}^K | & = \sum_{l=2}^L 2^{-l} \sqrt{ | \Psi_{h,l}^K | } \sqrt{ | \Psi_{h,l}^K | } \leq \sqrt{ 40 d L } \sum_{l=1}^L \sqrt{ | \Psi_{h,l}^K | }  \leq \sqrt{ 40 d L^2 \sum_{l=1}^L | \Psi_{h,l}^K | } \leq \sqrt{ 40 d L^2  K } .
\end{align}
Hence, combining the previous two inequalities, we have shown
\begin{equation} \label{eqRegBonus}
\sum_{l=2}^L 2^{ - l } | \Psi_{h,l}^K | \leq \min \{ 80 \cdot 2^L d L , \sqrt{ 40 d L^2 K } \} = \sqrt{ 40 d L^2  \min \{ 160 \cdot 2^{2L} d ,  K  \} } .
\end{equation}
Returning to \eqref{eqRegDecomp}, since $\{ \gamma_h^k \}_{k \in [K] h \in [H] }$ is a martingale difference sequence with $|\gamma_h^k| \leq 2H$, the Azuma-Hoeffding inequality implies that with probability at least $1-\delta/2$, 
\begin{equation} \label{eqRegMDS}
\sum_{k=1}^K \sum_{h=1}^H \gamma_h^k \leq \sqrt{ 8 H^3 K \log(2/\delta) }  .
\end{equation}
For the last term in \eqref{eqRegDecomp}, by definition $\epseff = 2 \alpha \sqrt{d \epsrnd } + 28 H \sqrt{ d L } \epsmis$ and $\epsrnd = 2^{-4L}/d$, and since $L \leq 2^{L-1}\ \forall\ L \in \N$, we have
\begin{equation} \label{eqRegEpsEff}
L \epseff  = 2^{ 1-2 L} L \alpha + 28 \sqrt{dL^3} H \epsmis \leq 2^{-L} \alpha + 28 \sqrt{ d L^3 } H \epsmis .
\end{equation}
Hence, when $\mathcal{G}$ and the event \eqref{eqRegMDS} both occur (which happens with probability at least $1-\delta$ by Lemma \ref{lemConcentration} and Azuma-Hoeffding), we can combine \eqref{eqRegDecomp}, \eqref{eqRegBonus}, \eqref{eqRegMDS}, and \eqref{eqRegEpsEff} to obtain
\begin{equation}
R(K) = O \left( \alpha \sqrt{ d  H^2 L^2 \min \{ 2^{2L} d , K \} } + \sqrt{H^3 K \log(\delta^{-1})} + 2^{-L } \alpha HK + \sqrt{ d   L^3   } H^2 K \epsmis \right)  .
\end{equation}
Recall $\alpha = 42 d H L \sqrt{ \log ( 3 d H L / \delta ) }$ and $\iota = \log(3dH/\delta)$, so $\alpha = O ( \sqrt{d^2 H^2 L^3 \iota} )$. Substituting above,
\begin{align}
R(K) = O \left( \sqrt{ d^3  H^4 L^5  \min \{ 2^{2L} d  , K \} \iota }  + \sqrt{H^3 K \iota}  + 2^{-L }  \sqrt{ d^2  L^3 \iota } H^2 K + \sqrt{ d   L^3   } H^2 K \epsmis \right)  .
\end{align}
The regret bound in Theorem \ref{thmMain} follows by definition $L = \ceil{ \log_2  ( \sqrt{d}/ \epstol) }$.

\section{Theorem \ref{thmMain} proof details} \label{appMainProofDetails}

\subsection{Regret bound details}

\begin{proof}[Proof of Claim \ref{clmTildeEpsrnd}]
The first bound follows from Holder's inequality, Algorithm \ref{algMain}, a standard norm equivalence, and Assumption \ref{assLin}:
\begin{equation}
| \phi(s,a)^\trans ( w_{h,l}^k - \tilde{w}_{h,l}^k ) | \leq \| \phi(s,a) \|_1 \| w_{h,l}^k - \tilde{w}_{h,l}^k \|_\infty \leq \sqrt{d} \| \phi(s,a) \|_2 \epsrnd \leq  \sqrt{d} \epsrnd .
\end{equation}
For the second bound, by similar logic, we have 
\begin{align}
| \phi(s,a)^\trans ( ( \Lambda_{h,l}^k )^{-1} - ( \tilde{\Lambda}_{h,l}^k )^{-1} ) \phi(s,a) | & \leq \sum_{i,j = 1}^d | \phi_i(s,a) | | \phi_j(s,a) | | ( ( \Lambda_{h,l}^k )^{-1} - ( \tilde{\Lambda}_{h,l}^k )^{-1} )_{ij} | \\
& \leq \| \phi(s,a) \|_1^2 \epsrnd \leq  d\| \phi(s,a) \|_2^2 \epsrnd \leq d \epsrnd ,
\end{align}
which implies that
\begin{align}
\| \phi(s,a) \|_{ (\Lambda_{h,l}^k )^{-1} } & \leq \sqrt{ | \phi(s,a)^\trans ( \tilde{\Lambda}_{h,l}^k )^{-1} \phi(s,a) | + | \phi(s,a)^\trans ( ( \Lambda_{h,l}^k )^{-1} - ( \tilde{\Lambda}_{h,l}^k )^{-1} ) \phi(s,a) | } \\
& \leq \| \phi(s,a) \|_{ (\tilde{\Lambda}_{h,l}^k )^{-1} } + \sqrt{ d \epsrnd } .
\end{align}
By symmetry, $\| \phi(s,a) \|_{ (\tilde{\Lambda}_{h,l}^k )^{-1} } \leq \| \phi(s,a) \|_{ (\Lambda_{h,l}^k )^{-1} } + \sqrt{d \epsrnd}$ as well, which completes the proof.
\end{proof}

\begin{clm} \label{clmEllPot}
For any $k \in [K]$, $h \in [H]$, and $l \in [L]$, we have
\begin{equation}
\sum_{\tau \in \Psi_{h,l}^k} \| \phi(s_h^\tau,a_h^\tau) \|_{ (\Lambda_{ h,l }^\tau )^{-1} }^2 \leq 2 \log ( \det  (\Lambda_{ h,l }^k ) / \det ( 16 I )  ) \leq 2 d \log ( 1 + | \Psi_{h,l}^k | / (16d) ) .
\end{equation}
\end{clm}
\begin{proof}
The first bound is a restatement of Lemma D.2 from \cite{jin2020provably}. The second follows from Lemma 10 of \cite{abbasi2011improved} and Assumption \ref{assLin}.
\end{proof}

\begin{proof}[Proof of Claim \ref{clmPsiBound}]
Since $\Psi_{h,l}^0 \subset \cdots \subset \Psi_{h,l}^K$, it suffices to prove the bound for $k = K$. By Algorithm \ref{algMain}, $\Psi_{h,l}^K$ is the set of episodes $k \in [K]$ for which $l_h^k(s_h^k) = l$. By Algorithm \ref{algSub}, $l_h^k(s_h^k) = l$ implies that $\| \phi(s_h^k,a_h^k) \|_{ (\tilde{\Lambda}_{ h,l }^k )^{-1} } > 2^{-l}$. Combined with Claim \ref{clmTildeEpsrnd}, and since $\epsrnd \leq 2^{-4l} / d \leq 2^{-2(l+1)} / d$ in Algorithm \ref{algMain}, we obtain
\begin{align}
1 & < 2^{2l} \| \phi(s_h^k,a_h^k) \|_{ (\tilde{\Lambda}_{ h,l }^k )^{-1} }^2 \leq 2^{2l+1} \| \phi(s_h^k,a_h^k) \|_{ ({\Lambda}_{ h,l }^k )^{-1} }^2 + 2^{2l+1} d \epsrnd  \leq 2^{2l+1} \| \phi(s_h^k,a_h^k) \|_{ ({\Lambda}_{ h,l }^k )^{-1} }^2 + 1/2 ,
\end{align}
or, after rearranging, $1 < 2^{2(l+1)} \| \phi(s_h^k,a_h^k) \|_{ ({\Lambda}_{ h,l }^k )^{-1} }^2$. Combined with Claim \ref{clmEllPot}, we obtain
\begin{equation} \label{eqPsiBoundInit}
|\Psi_{h,l}^K| = \sum_{k \in \Psi_{h,l}^K} 1 < 2^{2(l+1)} \sum_{k \in \Psi_{h,l}^K} \| \phi(s_h^k,a_h^k) \|_{ (\Lambda_{ h,l }^k )^{-1} }^2 \leq 2^{2l+3}  d \log ( 1 + |\Psi_{h,l}^K| / (16d) ) .
\end{equation}
Multiplying and dividing the right side of \eqref{eqPsiBoundInit} by $2$, we get
\begin{equation}
|\Psi_{h,l}^K| \leq 2^{2(l+2)} d \log \left( \sqrt{ 1 + |\Psi_{h,l}^K| / (16d) } \right) \leq 2^{2(l+2)} d \log \left( 1 + \sqrt{|\Psi_{h,l}^K| / (16d) } \right) \leq  2^{2 (l+1)} \sqrt{d |\Psi_{h,l}^K|} ,
\end{equation}
or, after rearranging,
\begin{equation} \label{eqPsiLooseBound}
|\Psi_{h,l}^K| \leq ( 2^{2(l+1)} \sqrt{d} )^2 = 2^{4(l+1)} d .
\end{equation}
Plugging \eqref{eqPsiLooseBound} into the right side of \eqref{eqPsiBoundInit}, we obtain
\begin{equation}
| \Psi_{h,l}^k | \leq 2^{2l+3} d \log ( 1 + 2^{4l} ) < 2^{2l+3} d \log ( 2 \cdot 2^{4l} ) \leq 2^{2l+3} d \log 2^{5l} < 2^{2l+3} d \cdot 5l = 40 \cdot 4^l d l . 
\end{equation}
Finally, since $40 < 64 = 2^6$ and $l \leq 2^{l-1}$ for any $l \in \N$, we have $40 \cdot 4^l  \cdot l < 2^6 \cdot 4^l \cdot 2^{l-1}  = 2^{3l+5}$.
\end{proof}

\begin{cor} \label{corDetRat}
For any $k \in [K]$, $h \in [H]$, and $l \in [L]$, we have $\det  (\Lambda_{ h,l }^k ) / \det ( 16 I ) \leq 2^{5dl}$.
\end{cor}
\begin{proof}
Combining Claims \ref{clmEllPot} and \ref{clmPsiBound}, we obtain
\begin{align}
 & \det  (\Lambda_{ h,l }^k ) / \det ( 16 I ) \leq ( 1 + | \Psi_{h,l}^k | / ( 16 d ) )^d \leq ( 1 + 2^{3 l + 1} )^d < ( 2 \cdot 2^{3l+1} )^d = 2^{ ( 3l+2) d } \leq 2^{5dl} . \qedhere
\end{align}
\end{proof}

\begin{cor} \label{corEllPot}
For any $k \in [K]$, $h \in [H]$, and $l \in [L]$, we have $\sum_{\tau \in \Psi_{h,l}^{k-1}} \| \phi(s_h^\tau,a_h^\tau) \|_{ (\Lambda_{ h,l }^k )^{-1} } \leq ( 2^l d )^4$.
\end{cor}
\begin{proof}
By Cauchy-Schwarz, Claim \ref{clmPsiBound}, and Lemma D.1 of \cite{jin2020provably},
\begin{align}
& \sum_{\tau \in \Psi_{h,l}^{k-1}} \| \phi(s_h^\tau,a_h^\tau) \|_{ (\Lambda_{ h,l }^k )^{-1} } \leq \sqrt{ | \Psi_{h,l}^{k-1} | \times \sum_{\tau \in \Psi_{h,l}^{k-1}} \| \phi(s_h^\tau,a_h^\tau) \|_{ (\Lambda_{ h,l }^k )^{-1} }^2 }  \leq \sqrt{ 2^{3l+5} d \times d } \leq 2^{ 4l } d \leq ( 2^l d )^4  . \qedhere
\end{align}
\end{proof}

\begin{clm} \label{clmNormEq}
For any $k \in [K]$, $h \in [H]$, $l \in [L]$, and $x \in \R^d$, we have $\| x \|_{ ( \Lambda_{h,l}^k )^{-1} } \leq \| x \|_2 / 4$.
\end{clm}
\begin{proof}
Let $\lambda_i \in [16,\infty)$ and $q_i \in \R^d$ denote the eigenvalues and eigenvectors of $\Lambda_{h,l}^k$. Then
\begin{align}
\| x \|_{ ( \Lambda_{h,l}^k )^{-1} }^2 & = \sum_{i=1}^d ( q_i^\trans x )^2 / \lambda_i \leq \sum_{i=1}^d ( q_i^\trans x )^2 / 16 = \| [q_1 \cdots q_d]^\trans x \|_2^2 / 16 = \| x \|_2^2 / 16 . 
\qedhere
\end{align}
\end{proof}

\begin{cor} \label{corExploratory}
For any $k \in [K]$, $h \in [H]$, and $s \in \S$, we have $l_h^k(s) \in \{2,\ldots,L+1\}$.
\end{cor}
\begin{proof}
Since $l_h^k(s) \in [L+1]$ by definition in Algorithm \ref{algSub}, it suffices to show $l_h^k(s) = 1$ cannot occur. If it does, then by Algorithm \ref{algSub}, Claim \ref{clmTildeEpsrnd}, Claim \ref{clmNormEq}, Assumption \ref{assLin}, and definition $\epsrnd = 2^{-4L} / d$, we can find $a \in \A$ to obtain the following contradiction:
\begin{align}
& 1/2 < \| \phi(s,a) \|_{ ( \tilde{\Lambda}_{h,l}^k )^{-1} } \leq \| \phi(s,a) \|_{ ( \Lambda_{h,l}^k )^{-1} } + \sqrt{d \epsrnd} \leq ( \| \phi(s,a) \|_2 / 4 ) + 2^{-2L} \leq 1/2 . \qedhere
\end{align}
\end{proof}

\begin{clm} \label{clmWl2Bound}
For any $k \in [K]$, $h \in [H]$, and $l \in [L]$, we have $\| {w}_{h,l}^k \|_2 \leq ( 2^l d H )^4$.
\end{clm}
\begin{proof}
For any $x \in \R^d$, Claim \ref{clmNormEq} implies
\begin{equation}
\| ( \Lambda_{h,l}^k )^{-1} x \|_2 = \| ( \Lambda_{h,l}^k )^{-1/2} x \|_{ ( \Lambda_{h,l}^k )^{-1} } \leq \| ( \Lambda_{h,l}^k )^{-1/2} x \|_2 / 4 = \| x \|_{ ( \Lambda_{h,l}^k )^{-1} } / 4 .
\end{equation}
Combined with the triangle inequality and Corollary \ref{corEllPot}, we obtain
\begin{align}
\| w_{h,l}^k \|_2 & \leq \frac{1}{4} \sum_{ \tau \in \Psi_{h,l}^{k-1} } \| \phi(s_h^\tau,a_h^\tau) \|_{ ( \Lambda_{h,l}^k )^{-1} } | r_h^\tau(s_h^\tau,a_h^\tau) + V_{h+1}^k ( s_{h+1}^\tau )  | \\
& < H \sum_{ \tau \in \Psi_{h,l}^{k-1} } \| \phi(s_h^\tau,a_h^\tau) \|_{ ( \Lambda_{h,l}^k )^{-1} }  \leq H ( 2^l d )^4 \leq ( 2^l d H )^4 ,
 \end{align}
where the second inequality holds because $r_h^\tau(s_h^\tau,a_h^\tau) \in [0,1]$ by assumption and $V_{h+1}^k ( s_{h+1}^\tau ) \in [0,H]$ by definition in Algorithm \ref{algSub} (so $| r_h^\tau(s_h^\tau,a_h^\tau) + V_{h+1}^k ( s_{h+1}^\tau )  | \leq 1+H\leq 2H < 4H$).
\end{proof}

\begin{clm} \label{clmBetaBound}
Define $\mathcal{V}$ as in the proof of Lemma \ref{lemConcentration}. Then for any $k \in [K]$, $h \in [H]$, and $l \in [L]$, we have
\begin{equation}
\sqrt{ 8 H^2 \log \left(  \frac{ \det ( \Lambda_{h,l}^k ) }{ \det ( 16 I )} \frac{2 H L |\mathcal{V}|}{\delta} \right) } \leq \beta .
\end{equation}
\end{clm}
\begin{proof}
We first bound $|\mathcal{V}|$. Clearly, $|\mathcal{V}| \leq |\mathcal{X}|^L |\mathcal{Y}|^L$. Next, observe
\begin{equation}
|\mathcal{Y}| = ( 1 + 2 \ceil{ 1 / ( 16 \epsrnd ) } )^{d^2} \leq ( 3 + 1 / ( 8 \epsrnd ) )^{d^2} < 1/  \epsrnd^{d^2} = 2^{4 L d^2} d^{d^2} ,
\end{equation}
where the second inequality holds since $\epsrnd < 7/24$ in Algorithm \ref{algMain}. For $\mathcal{X}$, we have
\begin{align}
|\mathcal{X}| & \leq ( 3 + 2 ( 2^L d H )^4 / \epsrnd )^d \leq ( ( 3 \cdot 2^{-4} + 2 ) ( 2^L d H )^4 / \epsrnd )^d < ( 4 ( 2^L d H )^4 / \epsrnd )^d = 2^{( 8 L + 2) d } d^{5d} H^{4d} ,
\end{align}
where the second inequality uses $\epsrnd \leq 1$ in Algorithm \ref{algMain}. Hence, we have shown
\begin{equation}
| \mathcal{V} | \leq  ( 2^{4 L d^2} d^{d^2} \cdot 2^{ (8L+2) d } d^{5d} H^{4d} )^L \leq ( 2^{14} d^6 H^4 )^{L^2 d^2} .
\end{equation}
Furthermore, by Corollary \ref{corDetRat}, we have $\det ( \Lambda_{h,l}^k ) / \det ( 16 I ) \leq  2^{5dl} \leq 2^{ 5 L^2 d^2 }$. Combining,
\begin{equation}
\log \left( \frac{2 H L}{\delta} \cdot \frac{ \det ( \Lambda_{h,l}^k ) |\mathcal{V}| }{ \det ( 16 I )} \right) \leq \log \left( \frac{ 2 H L  }{\delta } \cdot ( 2^{19} d^6 H^4 )^{L^2 d^2}  \right) \leq 20 L^2 d^2 \log ( 2 d H L / \delta ) .
\end{equation}
Together with the fact that $8 \cdot 20 = 160 < 169 = 13^2$, we obtain
\begin{align}
& \sqrt{ 8 H^2 \log \left(  \frac{ \det ( \Lambda_{h,l}^k ) }{ \det ( 16 I )} \frac{2 H L |\mathcal{V}|}{\delta} \right) } < 13 d H L \sqrt{ \log(3dHL/\delta) } = \beta . \qedhere
\end{align}
\end{proof}

\begin{proof}[Proof of Lemma \ref{lemEstErrorLearner}]
We fix $k$ and use induction on $h$. When $h = H$, since $V_{H+1}^k(\cdot) = 0$, we have $\gamma_H^k = 0$ and $V_H^{\pi_k}(s) = r_H ( s , \pi_H^k(s) ) = \bar{Q}_H^k(s , \pi_H^k(s) )$ (for any $s$). Therefore, by Claim \ref{clmEstErrorLearner}, for any $s$,
\begin{align}
V_H^k (s) - V_H^{\pi_k}(s) = V_H^k(s) - \bar{Q}_H^k(s , \pi_H^k(s) ) \leq 8 \alpha \cdot 2^{-l_H^k(s)}  + \epseff . 
\end{align}
Hence, choosing $s = s_H^k$ yields the bound. Now assume \eqref{eqRecursion} holds for $h+1 \in \{2,\ldots,H\}$. Then for any $s \in \S$,
\begin{align}
V_h^k(s) - V_h^{ \pi_k }(s)&= V_h^k(s) - \bar{Q}_h^k(s, \pi_h^k(s) ) + \int_{ s' \in \S } ( V_{h+1}^k(s') - V_{h+1}^{\pi_k}(s')  ) P_h( s' | s , \pi_h^k(s) ) \\
& \leq 8 \alpha \cdot 2^{-l_h^k(s)} + \epseff + \int_{ s' \in \S } ( V_{h+1}^k(s') - V_{h+1}^{\pi_k}(s')  ) P_h( s' | s , \pi_h^k(s) ) ,
\end{align}
where the inequality again uses Claim \ref{clmEstErrorLearner}. Choosing $s = s_h^k$ and using the fact that $a_h^k = \pi_h^k(s_h^k)$ and the inductive hypothesis, we thus obtain
\begin{align}
V_h^k(s_h^k) - V_h^{ \pi_k }(s_h^k) & \leq 8 \alpha \cdot 2^{-l_h^k(s_h^k)}  + \epseff + \gamma_h^k+ V_{h+1}^k(s_{h+1}^k) - V_{h+1}^{ \pi_k }(s_{h+1}^k) \\
& \leq 8 \alpha \sum_{h'=h}^H 2^{ - l_{h'}^k(s_{h'}^k) }  + \sum_{h'=h}^H \gamma_{h'}^k +  ( H - h + 1 ) \epseff . \qedhere
\end{align}
\end{proof}

\begin{proof}[Proof of Lemma \ref{lemOptimism}]
We again use induction on $h$. For $h = H$, the bound follows immediately from Claim \ref{clmOptimismH} and $V_{H+1}^\star(\cdot) = V_{H+1}^k(\cdot) = 0$. Now assume that for some $h+1 \in \{2,\ldots,H\}$ and all $s' \in \S$,
\begin{equation}
V_{h+1}^\star(s') - V_{h+1}^k(s')  \leq ( 2 L \epseff + 2^{-L} \alpha )  ( H - h )  .
\end{equation}
Combining this bound with Claim \ref{clmOptimismH}, for any $s \in \S$, we obtain
\begin{align}
V_h^\star(s) - V_h^k(s) & \leq \max \left\{ \max_{ a \in \A } \int_{s' \in \S} ( V_{h+1}^\star(s') - V_{h+1}^k(s') ) P_h(s'|s,a)  + 2 L \epseff + 2^{-L} \alpha , 0 \right\} \\
& \leq \max \left\{ ( 2 L \epseff  + 2^{-L} \alpha )(H-h+1) , 0 \right\} = ( 2 L \epseff  + 2^{-L} \alpha )  ( H - h + 1 ) . \qedhere
\end{align}
\end{proof}

\subsection{Complexity analysis}

The time complexity of Algorithm \ref{algSub} is dominated by the computation of the set of induced norms $\{ \| \phi(s,a) \|_{ ( \tilde{\Lambda}_{h,l}^k )^{-1} } : a \in \A_{h,l}^k(s) , l \in [ l_h^k(s) \wedge L ] \}$, which requires at most $O(d^2 L |\A|)$ time. At episode $k$ of Algorithm \ref{algMain}, Algorithm \ref{algSub} is called to compute $\{ V_{h+1}^k(s_{h+1}^\tau) : \tau \in \Psi_{h,l}^{k-1} , l \in [L] , h \in \{2,\ldots,H\} \}$ in Line \ref{lnWeightVectorUR} and $\{ \pi_h^k(s_h^k) : h \in [H] \}$ in Line \ref{lnTakeAction}, for a total number of calls
\begin{equation}
\sum_{h=2}^H \sum_{l=1}^L | \Psi_{h,l}^{k-1} | + H \leq 40 d H L \sum_{l=1}^L 4^l + H = O ( 4^L d H L ) ,
\end{equation}
where the inequality uses Claim \ref{clmPsiBound}. Alternatively, since $\cup_{l=1}^L \Psi_{h,l}^{k-1} \subset [K]$ for each $h \in [H]$, we can simply bound the number of calls by $O( HK )$. Hence, during each episode, the time complexity of all Algorithm \ref{algSub} calls is bounded above by
\begin{equation} \label{eqTimeCompInit}
O( d^2 H L |\A| \min\{ 4^L d L , K \} )  .
\end{equation}
Besides these calls, by a similar argument, computing the summations in Line \ref{lnWeightVectorUR} has runtime at most $d \min \{ \sum_{h=1}^H \sum_{l=1}^L | \Psi_{h,l}^{k-1} | , HK \}  = O ( d H  \min \{ 4^L d L  , K \} )$, which is dominated by \eqref{eqTimeCompInit}. Iterative updates of $\Lambda_{h,l}^k$ and $(\Lambda_{h,l}^k)^{-1}$ and the rounding in Lines \ref{lnLambdaR}-\ref{lnWeightVectorR} both have complexity $O(d^2 HL)$ per episode, and maintaining $\Psi_{h,l}^k$ has complexity $O(H)$; both quantities are dominated by \eqref{eqTimeCompInit}. Hence, Algorithm \ref{algMain}'s per-episode runtime is
\begin{equation}
O ( d^2 H L |\A| \min\{ 4^L d L , K \} ) = O  \left( d^2 H |\A|  \min \left\{ d^2 \log( \tfrac{d}{\epstol} ) / \epstol^2 , K \right\} \log ( \tfrac{d}{\epstol} ) \right) .
\end{equation}
At episode $k$, Algorithm \ref{algMain} uses $\{ \Lambda_{h,l}^k \}_{ h \in [H] , l \in [L] }$ and $\{ w_{h,l}^k \}_{ h \in [H] , l \in [L] }$, and their rounded versions, which requires $O(d^2 H L)$ storage. Note these can be overwritten across episodes, so the total storage is $O(d^2 H L)$ as well. Additionally, at episode $k$, it needs to store $r_h^\tau(s_h^\tau,a_h^\tau)$ and $\{ \phi(s_h^\tau,a) \}_{a \in \A}$, for each $\tau \in \Psi_{h,l}^{k-1}$, $h \in [H]$, and $l \in [L]$. Similar to above, this storage can be bounded by either $d|\A| HK$ or $d|\A| \sum_{h=1}^H \sum_{l=1}^L | \Psi_{h,l}^{k-1} | = O ( 4^L d^2 H L |\A|  )$. Hence, the total space complexity is
 \begin{equation}
O ( d^2 H L + d H |\A| \min \{ 4^L d L  , K \} ) = O \left( d^2 H \log(\tfrac{d}{\epstol}) + d H |\A| \min \left\{ d^2 \log ( \tfrac{d}{\epstol} ) / \epstol^2 , K \right\} \right) .
\end{equation}

\section{Proofs of propositions}

\subsection{Proposition \ref{propEnsemble} proof} \label{appEnsembleProof}

We use induction on $k$. For $k=1$, since $\Psi^0 = \Psi_l^0 = \emptyset$, we have $(w^k,\Lambda^k) = ( w_l^k , \Lambda_l^k ) = (0,0)$. For the inductive hypothesis, suppose $(w^k,\Lambda^k) = ( w_l^k , \Lambda_l^k )$ for some $k \in [K]$. By assumption, we can find $\tilde{s}^k \in \S$ such that
\begin{equation} \label{eqTildeSk}
\phi(\tilde{s}^k,a) = \begin{cases} \phi(s^k,a) , & k \in \Psi_l^K , a \in \A_l^k \\ 0 , & \text{otherwise} \end{cases} , \quad r(\tilde{s}^k,a) = r(s^k,a)\ \forall\ a \in \A . 
\end{equation}
Now consider two cases. First, if $k \notin \Psi_l^K$, then $( w_l^{k+1} , \Lambda_l^{k+1} ) = ( w_l^k , \Lambda_l^k )$. On the other hand, \eqref{eqTildeSk} implies $\max_{a \in \A} \| \phi(\tilde{s}^k,a)  \|_{ (\Lambda^k)^{-1} } = 0$, so $\Psi^k = \Psi^{k-1}$ and $( w^{k+1} , \Lambda^{k+1} ) = ( w^k , \Lambda^k )$. Hence, $( w^{k+1} , \Lambda^{k+1} ) = ( w_l^{k+1} , \Lambda_l^{k+1} )$ follows from the inductive hypothesis. Next, assume $k \in \Psi_l^K$. Then $\max_{a \in \A_l^k} \| \phi(s^k,a) \|_{(\Lambda_l^k)^{-1}} > 2^{-l}$ in \texttt{Sup-Lin-UCB-Var}, and it plays $a^k = \argmax_{a \in \A_l^k} \| \phi(s^k,a) \|_{(\Lambda_l^k)^{-1}}$, observes $r^k(s^k,a^k) = r(s^k,a^k) + \eta^k$, and updates
\begin{gather}
\Lambda_l^{k+1} = I + \sum_{ \tau \in \Psi_l^k } \phi(s^\tau,a^\tau) \phi(s^\tau,a^\tau)^\trans = \Lambda_l^k + \phi(s^k,a^k) \phi(s^k,a^k)^\trans , \label{eqSLUlam} \\
 w_l^{k+1} = ( \Lambda_l^{k+1} )^{-1} \sum_{ \tau \in \Psi_l^k } \phi(s^\tau,a^\tau) r^\tau(s^\tau,a^\tau) = ( \Lambda_l^{k+1} )^{-1} ( \Lambda_l^k w_l^k + \phi(s^k,a^k) r^k(s^k,a^k) ) . \label{eqSLUw} 
\end{gather}
On the other hand, $\max_{a \in \A} \| \phi(\tilde{s}^k,a) \|_{(\Lambda^k)^{-1}} = \max_{a \in \A_l^k} \| \phi(s^k,a) \|_{(\Lambda_l^k)^{-1}} > 2^{-l}$ by \eqref{eqTildeSk} and the inductive hypothesis. Hence, if we write $\tilde{a}^k$ (instead of $a^k$) for the action chosen by $\texttt{EXPL}(2^{-l})$, we have
\begin{equation}
\tilde{a}^k = \argmax_{a \in \A} \| \phi(\tilde{s}^k,a) \|_{(\Lambda^k)^{-1}} = \argmax_{a \in \A_l^k} \| \phi(s^k,a) \|_{(\Lambda_l^k)^{-1}} = a^k .
\end{equation}
By \eqref{eqTildeSk}, this implies $\phi(\tilde{s}^k,\tilde{a}^k) = \phi(s^k,a^k)$, and by \eqref{eqTildeSk} and the noise assumption, we also have $r^k(\tilde{s}^k,\tilde{a}^k) = r^k(s^k,a^k)$. Hence, similar to \eqref{eqSLUlam} and \eqref{eqSLUw}, we see that $\texttt{EXPL}(2^{-l})$ updates
\begin{equation} \label{eqSLUtilde}
\Lambda^{k+1} = \Lambda^k + \phi(s^k,a^k) \phi(s^k,a^k)^\trans , \quad w^{k+1} = ( \Lambda^{k+1} )^{-1} ( \Lambda^k w^k + \phi(s^k,a^k) r^k(s^k,a^k) . 
\end{equation}
Combining \eqref{eqSLUlam}, \eqref{eqSLUw}, \eqref{eqSLUtilde}, and the inductive hypothesis completes the proof.

\subsection{Proposition \ref{propLinUcb} proof} \label{appLinUcbProof}

We essentially follow the existing proof . First, define the good event
\begin{equation}
\mathcal{G}' = \cap_{k=1}^K \left\{ \left\| \sum_{\tau=1}^{k-1} \phi(s^\tau,a^\tau) \eta^\tau \right\|_{ ( \Lambda^k )^{-1} } \leq \sqrt{ 2 d \log ( (\lambda+K)/(\lambda \delta ) ) } \right\} .
\end{equation}
By Theorem 1 of \cite{abbasi2011improved}, the assumption on $\eta^\tau$, Lemma 10 of \cite{abbasi2011improved}, and the $\ell_2$ norm assumption on $\phi(s^\tau,a^\tau)$, $\P(\mathcal{G}') \geq 1-\delta$. We bound regret on $\mathcal{G}'$. Let $\Delta^\tau = r(s^\tau,a^\tau) - \phi(s^\tau,a^\tau)^\trans \theta$. Then if we define
\begin{align}
w_{k,1} = ( \Lambda^k )^{-1} \sum_{\tau=1}^{k-1} \phi(s^\tau,a^\tau)  \phi(s^\tau,a^\tau)^\trans \theta  , \quad w_{k,2} = ( \Lambda^k )^{-1} \sum_{\tau=1}^{k-1} \phi(s^\tau,a^\tau) \eta^\tau , \quad w_{k,3} = ( \Lambda^k )^{-1} \sum_{\tau=1}^{k-1} \phi(s^\tau,a^\tau)  \Delta^\tau ,
\end{align}
we have $w_k = \sum_{i=1}^3 w_{k,i}$, which implies
\begin{equation}
| \phi(s,a)^\trans ( w_k - \theta ) | \leq | \phi(s,a)^\trans ( w_{k,1} - \theta ) | + | \phi(s,a)^\trans w_{k,2} | + | \phi(s,a)^\trans w_{k,3} | .
\end{equation}
Similar to the analysis of $z_2$ in the proof of Claim \ref{clmEstError}, we have
\begin{equation} \label{eqW1final}
| \phi(s,a)^\trans ( w_{k,1} - \theta ) | = \lambda | \phi(s,a)^\trans ( \Lambda^k )^{-1} \theta | \leq \sqrt{\lambda} \| \phi(s,a) \|_{( \Lambda^k )^{-1}} .
\end{equation}
On the event $\mathcal{G}'$, we obtain
\begin{equation} \label{eqW2final}
| \phi(s,a)^\trans w_{k,2} | \leq \sqrt{ 2 d \log ( ( \lambda + K ) / ( \lambda \delta ) ) } \| \phi(s,a) \|_{( \Lambda^k )^{-1}} .
\end{equation}
Similar to analysis of $z_3$ in the proof of Claim \ref{clmEstError}, we know
\begin{align}
| \phi(s,a)^\trans w_{k,3} | \leq \epsmis \sqrt{K} \| \phi(s,a) \|_{( \Lambda^k )^{-1}}  .
\end{align}
Hence, combining the previous four inequalities, we have shown that on $\mathcal{G}'$,
\begin{equation}
| \phi(s,a)^\trans ( w_k - \theta ) | \leq \left( \sqrt{\lambda} + \sqrt{ 2 d \log ( ( \lambda + K ) / ( \lambda \delta ) ) }  + \epsmis \sqrt{K} \right) \| \phi(s,a) \|_{( \Lambda^k )^{-1}} \leq \alpha \| \phi(s,a) \|_{( \Lambda^k )^{-1}} ,
\end{equation}
where the second inequality is by choice of $\lambda$ and $\alpha$. Thus, by the misspecification assumption,
\begin{equation}
| r(s,a) - \phi(s,a)^\trans w_k | \leq | r(s,a) - \phi(s,a)^\trans \theta | + | \phi(s,a)^\trans ( \theta - w_k ) | \leq \epsmis + \alpha \| \phi(s,a) \|_{( \Lambda^k )^{-1}}  .
\end{equation}
Hence, by definition of the optimal policy and \texttt{Lin-UCB},
\begin{equation}
r(s^k,a_\star^k) \leq \max_{a \in \A} \left( \phi(s^k,a)^\trans w_k + \alpha \| \phi(s^k,a) \|_{( \Lambda^k )^{-1}} + \epsmis  \right)   \leq r(s^k,a^k) +  2 \alpha \| \phi(s^k,a^k) \|_{( \Lambda^k )^{-1}} + 2 \epsmis . 
\end{equation}
Substituting into the regret definition, then using Cauchy-Schwarz, Lemma 11 from \cite{abbasi2011improved}, and some simple eigenvalue bounds, we obtain
\begin{equation}
R(K) \leq 2 \alpha \sum_{k=1}^K \| \phi(s^k,a^k) \|_{( \Lambda^k )^{-1}} + 2 \epsmis K  \leq 2 \alpha \sqrt{ K d \log ( (\lambda+K)/K) } + 2 \epsmis K .
\end{equation}
Since $( \lambda + K ) / \lambda = 1 + K / ( 1 + \epsmis^2 K ) \leq 1 + \min \{ K , \epsmis^{-2} \}$, by definition of $\alpha$, we have
\begin{equation}
\alpha = O \left( \sqrt{ 2 d \log ( \min \{ K , \epsmis^{-2} \} / \delta ) } + \epsmis \sqrt{K} \right) .
\end{equation}
Combining the previous two bounds and again using $( \lambda + K ) / \lambda \leq 1 + \min \{ K , \epsmis^{-2} \}$ yields the desired bound.

\end{document}